\documentclass{article}

\PassOptionsToPackage{numbers, compress}{natbib}

\usepackage[preprint]{staix_2026}

\usepackage[utf8]{inputenc}
\usepackage[T1]{fontenc}
\usepackage{hyperref}
\usepackage{url}
\usepackage{booktabs}
\usepackage{amsfonts}
\usepackage{nicefrac}
\usepackage{microtype}
\usepackage{graphicx}
\usepackage{amsmath, amssymb, amsthm, mathtools}
\usepackage{bm}
\usepackage{multirow}
\usepackage{enumitem}
\usepackage{algorithm}
\usepackage{algpseudocode}
\usepackage{graphicx}
\usepackage{multirow}
\usepackage{booktabs}
\usepackage{siunitx}
\usepackage{makecell}
\usepackage[table]{xcolor}
\definecolor{lightgreen}{RGB}{220,245,220}
\definecolor{askgreen}{RGB}{221,245,224}   
\definecolor{skipgray}{RGB}{238,238,238}   
\definecolor{missred}{RGB}{255,230,230}    
\definecolor{obsblue}{RGB}{225,240,255}    
\definecolor{warnyellow}{RGB}{255,245,204} 
\DeclareMathOperator*{\argmin}{arg\,min}

\usepackage{titlesec}
\titlespacing*{\section}{0pt}{0.8ex plus 0.2ex minus 0.1ex}{0.4ex plus 0.1ex}
\titlespacing*{\subsection}{0pt}{0.6ex plus 0.2ex minus 0.1ex}{0.3ex plus 0.1ex}
\titlespacing*{\subsubsection}{0pt}{0.5ex plus 0.1ex minus 0.1ex}{0.2ex plus 0.1ex}
\titlespacing*{\paragraph}{0pt}{0.4ex plus 0.1ex minus 0.1ex}{0.5em}

\newcommand{\R}{\mathbb{R}}
\newcommand{\ind}{\mathbb{I}}
\newcommand{\E}{\mathbb{E}}

\newcommand{\argmax}{\mathrm{argmax}}

\newtheorem{proposition}{Proposition}

\title{TabSODA: Tabular Diffusion based Imputation with\\
Skip Pattern Detection and Ordinal Awareness}

\author{%
  Yuyu Chen \\
  Department of Biostatistics \\
  NYU School of Global Public Health \\
  \texttt{yuyu.chen@nyu.edu}
  \And
  Taehyo Kim \\
  Department of Biostatistics \\
  NYU School of Global Public Health \\
  \texttt{tk2737@nyu.edu}
  \And
  Hai Shu \\
  Department of Biostatistics \\
  NYU School of Global Public Health \\
  \texttt{hs120@nyu.edu}
  \And
  Yang Feng \\
  Department of Biostatistics \\
  NYU School of Global Public Health \\
  \texttt{yang.feng@nyu.edu}
}

\begin{document}
\maketitle

\begin{abstract}
Missing data imputation in large-scale surveys faces two challenges that are not well handled by current tabular diffusion methods. First, \emph{structural skips}, cells made inapplicable by questionnaire design, should not be imputed but are often conflated with item nonresponse. Second, \emph{ordinal} responses encode ordered categories, yet most pipelines treat them as nominal levels through one-hot or analog-bit encodings. We introduce \textbf{TabSODA} (\textbf{Tab}ular diffusion with \textbf{S}kip pattern detection and \textbf{O}r\textbf{d}inal \textbf{A}wareness), an Expectation-Maximization (EM)-based diffusion imputer built on the Elucidated Diffusion Model (EDM) framework. TabSODA propagates structural skips through the denoising loss and reverse-time sampler, and represents ordinal variables with cumulative-probit scalar latents while retaining analog-bit encodings for nominal variables. When a codebook skip mask is available, TabSODA uses it directly; otherwise, the TabSODA+SKIP variant estimates the mask from raw responses and questionnaire order using a CART-based skip-pattern miner. On Population Assessment of Tobacco and Health (PATH) study and the National Survey on Drug Use and Health (NSDUH), two nationally representative U.S.\ surveys, TabSODA reduces ordinal MACE by up to $23.7\%$ and improves categorical accuracy by up to $9\%$ over the strongest baseline across MCAR, MAR, and MNAR masking. The skip miner achieves near-perfect precision on both datasets, allowing TabSODA+SKIP to closely track the codebook-mask variant. The code is available at \href{https://anonymous.4open.science/r/TabSODA-F3C4}{https://anonymous/TabSODA}.
\end{abstract}


\section{Introduction}
\label{sec:intro}

\vspace{-0.2cm} 

Missing data are common in surveys for health and social-science studies, and how they are handled can affect the quality of regression estimates, prevalence estimates, uncertainty intervals, and downstream prediction performance~\cite{rubin1976,littlerubin2002,carpenter2021missing}. Therefore, accurate imputation that accounts for the data-generating process, the measurement scale of each variable, and the reason a cell is empty is necessary. 

Imputation methods broadly fall into conditional and joint generative approaches. Conditional methods predict each missing entry from observed covariates, as in MICE~\cite{vanbuuren2011mice}, MissForest~\cite{stekhoven2012missforest}, HyperImpute~\cite{jarrett2022hyperimpute}, the masked autoencoder imputer ReMasker~\cite{du2024remasker}, and the graph-based imputer GRAPE~\cite{you2020grape}. Joint generative methods instead model the full-data distribution and sample missing entries conditional on observed cells, using Generative Adversarial Networks (GANs)~\cite{yoon2018gain}, Variational Autoencoders (VAEs)~\cite{mattei2019miwae}, normalizing flows~\cite{richardson2020mcflow}, or score-based diffusion models~\cite{song2021scoresde}. Recent diffusion-based imputers follow this second view. DiffPuter~\cite{zhang2025diffputer} introduces an Expectation-Maximization (EM)-style alternation between complete-data density estimation and conditional reverse diffusion that we adopt as our backbone.

Despite this progress, survey imputation still faces two gaps. First, large survey programs such as the Population Assessment of Tobacco and Health (PATH) study~\cite{icpsrpathpuf} and the National Survey on Drug Use and Health (NSDUH)~\cite{samhsansduhdata} contain item nonresponse, respondent breakoff, and \emph{structural skips}: blank cells induced by questionnaire routing, such as tobacco-frequency follow-ups for respondents who report never using tobacco. Treating all blanks as imputable can produce out-of-domain values, distort item denominators, and violate skip logic. Second, many survey variables are ordinal, including Likert scales, frequency groups, symptom scales, and severity ratings. Generic tabular pipelines often encode these columns as one-hot or analog-bit categorical variables, treating ordered categories as nominal labels and losing adjacency information~\cite{chen2023analogbits}. In contrast, cumulative-link models~\cite{mccullagh1980,christensen2025clm}, deep cumulative-link classifiers~\cite{vargas2020clm}, and rank-consistent ordinal networks~\cite{cao2020coral,nxumalo2025ordinal} use ordered thresholds or monotone targets and improve over nominal classifiers on ordered-label~tasks.

We propose \textbf{TabSODA} (\textbf{Tab}ular diffusion with \textbf{S}kip pattern detection and \textbf{O}r\textbf{d}inal \textbf{A}wareness), a skip-aware and ordinal-aware diffusion imputer for mixed-type survey data. TabSODA builds on the EM diffusion backbone of DiffPuter \cite{zhang2025diffputer}, but changes the data encoding-decoding scheme,  denoising loss, and reverse-time sampler to preserve questionnaire structure and ordinal scale during training and imputation.  The proposed method aims to recover much of the survey-design information available from a codebook skip mask while requiring only the raw response table and questionnaire order. Our contributions are summarized as follows. First, TabSODA integrates structural skips into EM-based diffusion imputation by propagating skip states through the denoising loss, EDM input, and reverse-time sampler, rather than treating skip handling as a preprocessing step. Second, it introduces a decision tree-based skip-pattern miner that estimates three-state masks from raw missingness patterns and questionnaire order, along with a codebook reference track for evaluating skip-mask quality; this reduces reliance on hand-coded skip rules when codebooks are incomplete, unavailable, or difficult to use. Third, it introduces an ordinal-aware mixed-type encoding in which each ordered categorical column is represented by a cumulative-probit scalar latent variable that preserves category adjacency, with an analog-bit fallback for columns dominated by a single category.


\section{Related Work}
\label{sec:related}
\vspace{-0.2cm} 

\paragraph{Missing data imputation.}
Conditional imputers predict missing cells from observed covariates: MICE~\cite{vanbuuren2011mice} fits chained regressions and MissForest~\cite{stekhoven2012missforest} uses iterative random forests. Joint generative methods instead estimate the full-data distribution and impute by conditioning on the observed values. GAIN~\cite{yoon2018gain} uses adversarial training, MIWAE~\cite{mattei2019miwae} and HI-VAE~\cite{nazabal2020hivae} use variational latent-variable models, and CSDI~\cite{tashiro2021csdi}, TabCSDI~\cite{zheng2022tabcsdi}, MissDiff~\cite{ouyang2025missdiff}, TabDiff~\cite{shi2024tabdiff}, and TabSyn~\cite{zhang2023tabsyn} learn conditional or joint score models for imputation or synthesis. DiffPuter~\cite{zhang2025diffputer} alternates between diffusion-based density estimation on the current completed table and conditional resampling of missing cells, making it the direct methodological baseline for TabSODA. These methods typically use a binary observed/missing mask and standard categorical encodings, so they do not distinguish ordinary nonresponse from structural nonapplicability.

\paragraph{Representation of skip patterns.}
Survey methodology has long recognized the distinction between skip-induced blank cells and ordinary item nonresponse. Graph-theoretic methods~\cite{fagan1988graph} represent a questionnaire as a routing graph that separates valid responses, nonapplicable cells, imputable missing cells, and unresolved cells, and later extensions handle clinical survey data~\cite{arslanturk2012skip,arslanturk2016survey}. Data-mining methods~\cite{wang2010classification} learn questionnaire rules or informative missingness patterns directly from observed data. Survey-imputation studies have shown that variables defined by skip patterns require separate treatment from item nonresponse~\cite{judkins2007skippreservation,zhang2023skippatternmi}. TabSODA brings this distinction inside a generative imputer. It mines a structural skip mask from raw missingness and questionnaire order, then passes that mask to the diffusion denoising loss and sampler.

\paragraph{Modeling of ordinal variables.}
Ordinal variables are common in survey instruments, particularly Likert scales, frequency bands, symptom scales, and severity ratings, where between-category ordering carries information that nominal classifiers discard. The proportional-odds cumulative-link model of McCullagh~\cite{mccullagh1980} and its modern treatments~\cite{christensen2025clm} represent ordered categories through a thresholded latent score, providing a likelihood that respects category adjacency. Deep extensions train neural networks under cumulative-link likelihoods~\cite{vargas2020clm} or impose rank-consistent classifier ensembles~\cite{cao2020coral,nxumalo2025ordinal}, both of which empirically outperform plain softmax classifiers on ordered-label tasks. TabSODA imports the cumulative-probit observation model into a diffusion imputer. Each ordinal column is represented by a single scalar latent that the denoiser predicts, and a per-column cutpoint vector decodes the prediction into category probabilities, so the diffusion process operates in a space where adjacent ordinal categories are closer than non-adjacent ones.

\begin{figure}[t]
\centering
\includegraphics[width=\textwidth,trim=0 8.18in 0 0,clip]{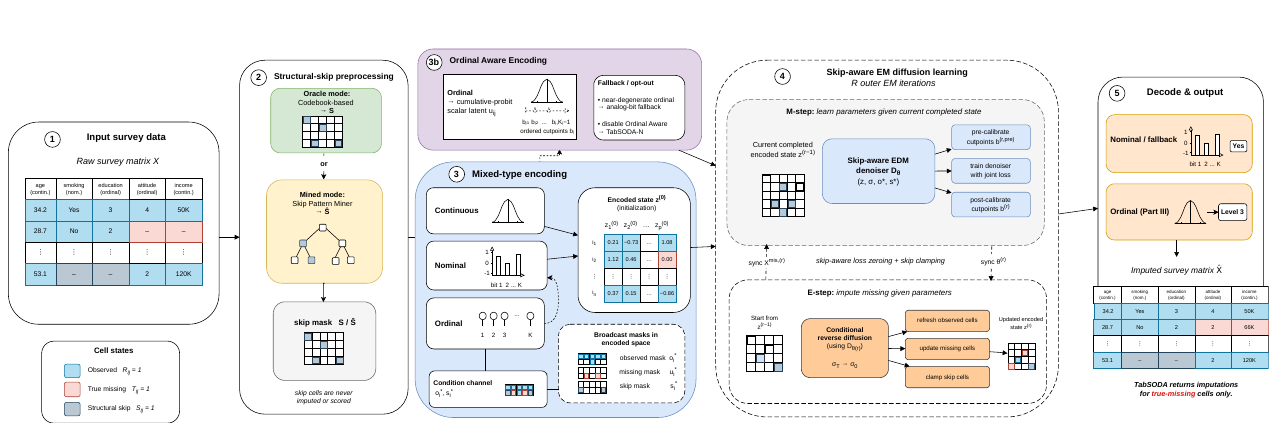}
\vspace{-1.2em}
\caption{An overview of the proposed TabSODA architecture. Structural-skip mask preprocessing uses either a codebook-provided mask or a learned skip-pattern miner, followed by a mixed-type encoding trained by a skip-aware EM diffusion loop. Ordinal columns route through a cumulative-probit scalar latent with blockwise cutpoint calibration with hybrid analog-bit fallback.}
\label{fig:pipeline}
\end{figure}

\vspace{0.2cm} 

\section{Method}
\label{sec:method}

\subsection{Problem Formulation}
\label{sec:method:problem}

\vspace{-0.2cm} 
Consider a survey data table $\mathbf{X}=(x_{ij}) \in \mathbb{R}^{N \times d}$ with \(N\) respondents and \(d\) variables, where \(x_{ij}\) is respondent \(i\)'s response to question \(j\) when observed. Variables may be continuous, nominal, or ordinal, with index sets \(\mathcal J_c\), \(\mathcal J_n\), and \(\mathcal J_o\), respectively. For an ordinal variable
with index \(j\in\mathcal J_o\), categories are ordered as \(1<\cdots<K_j\), as illustrated in Table~\ref{tab:ordinal_example}. Because ordinal levels encode relative severity or intensity, TabSODA preserves their ordered structure during imputation rather than treating them as nominal categories.

A central challenge in survey imputation is that questionnaire routing can make blank cells structurally inapplicable rather than missing. For example, in PATH~\cite{icpsrpathpuf}, answering \texttt{No} to ``Have you ever smoked a cigarette, even one or two puffs?'' (\texttt{AC1002}) skips downstream smoking follow-ups, producing a block of structural blanks; Table~\ref{tab:structural_skip_example} in Appendix~\ref{app:skip_mining} illustrates this pattern. Imputing such cells inflates missingness and can produce invalid responses, so TabSODA treats structural skips as non-imputable. Let \(R,T,S\in\{0,1\}^{N\times d}\) denote the observed, true-missing, and structural-skip masks, with entries \(R_{ij}\), \(T_{ij}\), and \(S_{ij}\). These masks are mutually exclusive and exhaustive, so \(R_{ij}+T_{ij}+S_{ij}=1\) for all \((i,j)\). Cells with \(T_{ij}=1\) are imputation targets, whereas cells with \(S_{ij}=1\) remain fixed as inapplicable. When codebook routing is available, \(S\) is obtained directly; otherwise, TabSODA uses an estimated mask \(\widehat S\) learned from raw missingness and questionnaire order.

The remainder of this section describes the components of TabSODA. Section~\ref{sec:method:skip_learning} presents the decision tree-based skip-pattern miner for estimating $\widehat{S}$. Section~\ref{sec:method:ordinalencoding} defines the ordinal-aware mixed-type representation used by the diffusion model. Section~\ref{sec:method:diffusion_em} combines the above modules to introduce a skip-pattern informed diffusion imputation model based on EM. 

\begin{table}[h]
\centering
\small
\caption{An example of an ordinal variable (AG1005FC from PATH~\cite{icpsrpathpuf}). Quantity of filtered cigars ever smoked is recorded using ordered categories, larger values indicating higher~consumption.}
\label{tab:ordinal_example}
\renewcommand{\arraystretch}{1.15}
\resizebox{\columnwidth}{!}{%
\begin{tabular}{ccccccc}
\toprule
Category & 1 & 2 & 3 & 4 & 5 & 6 \\
\midrule
Meaning & 
less than 1 &
1-10 cigars &
11-20 cigars &
21-50 cigars &
51-99 cigars &
100 or more \\
\bottomrule
\end{tabular}%
}
\end{table}

\subsection{Skip Pattern Mining based on Decision Trees}
\label{sec:method:skip_learning}

\vspace{-0.2cm} 

\begingroup
\setlength{\abovedisplayskip}{4pt}
\setlength{\belowdisplayskip}{4pt}
\setlength{\abovedisplayshortskip}{2pt}
\setlength{\belowdisplayshortskip}{2pt}

When the codebook skip mask \(S\) is unavailable, TabSODA uses the Classification and Regression Trees (CART)~\cite{breiman1984cart} to estimate it from the raw response table and the questionnaire order. Figure~\ref{fig:path50_cigarette_skip_cascade} illustrates an example of why learned skip states are
needed. In the PATH dataset, \texttt{AC1002} asks whether the respondent has ever smoked a
cigarette. If the response is \texttt{No}, \texttt{AC1003} is structurally
skipped. A later recency item, \texttt{AC0100}, can then be skipped either
because \texttt{AC1003} = \texttt{Not at all} or because \texttt{AC1003} was
already skipped. The latter case is not recoverable from the raw value of
\texttt{AC1003}; it requires the learned state feature
\texttt{AC1003\_STATE}.

Let \(\widehat S\in\{0,1\}^{N\times d}\) denote the mined skip mask, where its entry \(\widehat S_{it}=1\) indicates that respondent \(i\)'s cell for item \(q_t\) is predicted to be structurally skipped. \(q_t\) is the survey item at position \(t\) in the questionnaire order. The miner separates raw blank cells into structurally inapplicable cells, which are excluded from imputation, and unresolved missing cells, which remain imputation targets. Let \(q_1,\ldots,q_d\) be the questionnaire-ordered target variables after respondent identifiers are removed, and let \(x_{it}\) be respondent \(i\)'s raw entry for item \(q_t\). The initial and final states are
\begin{equation}\label{eq:skip_states}
  \mathrm{state}^{(0)}_{it}=\mathbb I\{x_{it}\text{ is raw-missing}\},\qquad
  \mathrm{state}_{it}=\begin{cases}0, & x_{it}\text{ is observed},\\ 1+\widehat S_{it}, & x_{it}\text{ is raw-missing}.\end{cases}
\end{equation}
Here \(\widehat S_{it}=1\) only for raw-missing cells covered by accepted skip rules. Thus, \(\mathrm{state}_{it}=1\) denotes unresolved missingness and \(\mathrm{state}_{it}=2\) denotes a mined structural skip.

Targets are processed in questionnaire order. For target \(q_t\), predictors use only earlier items \(q_g\), \(g<t\): categorical encodings of earlier raw responses, raw-missingness indicators, and learned earlier states \(\mathrm{state}_{ig}\in\{0,1,2\}\). This forward-only construction prevents later variables from explaining earlier skips while allowing learned skip states to propagate through skip cascades. For each target $q_t$, the CART label is raw missingness $y_{it}=\mathbb I\{x_{it}\text{ is raw-missing}\}=\mathbb I\{\mathrm{state}^{(0)}_{it}=1\}$, where $\mathbb I\{\cdot\}$ is the indicator function. Since $S_{it}$ is not observed during mining, $y_{it}=1$ means only that the target cell is blank: it may be a structural skip, item nonresponse, breakoff, or another unresolved missing value. The CART thus serves as a high-confidence rule generator rather than a classifier that converts all blank cells to skips.

A candidate CART rule \(r\) induces a prediction mask \(r_{it}\in\{0,1\}\), where \(r_{it}=1\) means respondent \(i\) satisfies the rule path for target \(q_t\); a schematic of the one-target learning flow is given in Figure~\ref{fig:skip_detection_flow} of Appendix~\ref{app:skip_mining}. For scoring rows \(\mathcal I_t\), the rule produces the $2\times 2$ confusion counts
\[
\setlength{\arraycolsep}{3pt}
\renewcommand{\arraystretch}{1.0}
\begin{array}{c|cc}
 & y_{it}=1 & y_{it}=0 \\
\hline
r_{it}=1 & \mathrm{TP}_t & \mathrm{FP}_t \\
r_{it}=0 & \mathrm{FN}_t & \mathrm{TN}_t
\end{array},
\]
with the number of real positives $n_+(t)=\mathrm{TP}_t+\mathrm{FN}_t$, the number of real negative $n_-(t)=\mathrm{FP}_t+\mathrm{TN}_t$, and empirical rates $\widehat{\mathrm{FPR}}_t=\mathrm{FP}_t/n_-(t)$ and $\widehat{\mathrm{FNR}}_t=\mathrm{FN}_t/n_+(t)$. The false-positive rate controls rule firing on nonblank target cells, while the false-negative rate controls missed raw blanks; targets with \(n_+(t)=0\) or \(n_-(t)=0\) are excluded from rule acceptance. Rather than screening by raw proportions, TabSODA uses Wilson upper confidence bounds, denoted \(\mathrm{UCB}(e,n)\), for an error count \(e\) among \(n\) relevant rows. A candidate rule is eligible only if
\begin{equation}\label{eq:skip_rule_screen_main}
  \mathrm{UCB}\{\mathrm{FP}_t,n_-(t)\}\le f_{\mathrm{FPR}}^{(t)}\quad\text{and}\quad \mathrm{UCB}\{\mathrm{FN}_t,n_+(t)\}\le f_{\mathrm{FNR,cap}}^{(t)}.
\end{equation}
The cap \(f_{\mathrm{FPR}}^{(t)}\) limits firing on nonblank rows, and \(f_{\mathrm{FNR,cap}}^{(t)}\) limits missed raw blanks with a small correction for sparse positive classes. Their position-adjusted definitions and the exact Wilson formula are given in Appendix~\ref{app:skip_mining}. Rules passing \eqref{eq:skip_rule_screen_main} are then tested for association with \(y_t\) using Fisher's exact test~\cite{fisher1934exact} on the same \(2\times2\) table, and Benjamini-Hochberg FDR control~\cite{benjamini1995fdr} is applied to the resulting \(p\)-values. Let \(\mathcal R_t\) be the accepted rule set for target $q_t$; the mined skip estimate is $\widehat S_{it}=\mathbb I\{y_{it}=1\text{ and }\exists\, r\in\mathcal R_t\text{ with }r_{it}=1\}$. Thus, only raw blank cells covered by accepted rules are promoted to structural skips; observed cells are never overwritten, and uncovered raw blanks remain unresolved missing values.
\endgroup

\begin{figure}[t]
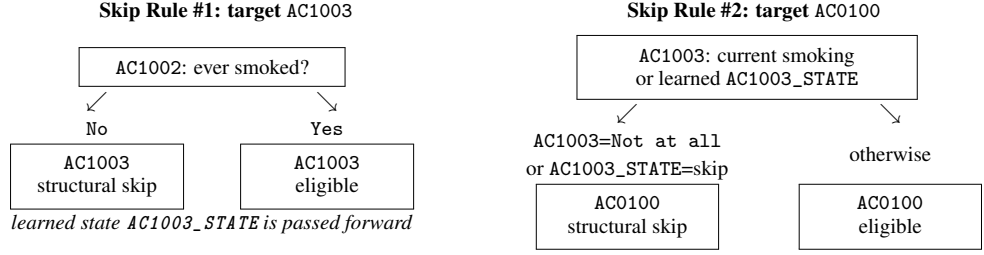

\centering
\small
\setlength{\fboxsep}{5pt}
\renewcommand{\arraystretch}{1.2}
\resizebox{\columnwidth}{!}{%
\begin{tabular}{c@{\qquad}c}
\multicolumn{1}{c}{\textbf{Skip Rule \#1: target \texttt{AC1003}}}
&
\multicolumn{1}{c}{\textbf{Skip Rule \#2: target \texttt{AC0100}}}
\\[4pt]

\begin{tabular}{c}
\fbox{\parbox{0.26\textwidth}{\centering \texttt{AC1002}: ever smoked?}} \\[3pt]
\begin{tabular}{c c c}
$\swarrow$ & & $\searrow$ \\
\texttt{No} & & \texttt{Yes} \\
\fbox{\parbox{0.16\textwidth}{\centering \texttt{AC1003}\\structural skip}} 
& &
\fbox{\parbox{0.16\textwidth}{\centering \texttt{AC1003}\\eligible}}
\end{tabular}
\\[5pt]
\textit{learned state \texttt{AC1003\_STATE} is passed forward}
\end{tabular}
&
\begin{tabular}{c}
\fbox{\parbox{0.34\textwidth}{\centering
\texttt{AC1003}: current smoking\\
or learned \texttt{AC1003\_STATE}}} \\[3pt]
\begin{tabular}{c c c}
$\swarrow$ & & $\searrow$ \\
\begin{tabular}{c}
\texttt{AC1003}=\texttt{Not at all}\\
or \texttt{AC1003\_STATE}=skip
\end{tabular}
& &
otherwise \\
\fbox{\parbox{0.17\textwidth}{\centering \texttt{AC0100}\\structural skip}}
& &
\fbox{\parbox{0.17\textwidth}{\centering \texttt{AC0100}\\eligible}}
\end{tabular}
\end{tabular}
\end{tabular}%
}
\caption{Two-step cigarette-use skip cascade learned by CART on PATH dataset. The second rule depends on both a raw response and the learned skip state of an earlier item.}
\label{fig:path50_cigarette_skip_cascade}
\end{figure}

\vspace{-0.2cm} 

\subsection{Ordinal-Aware Mixed-Type Encoding}
\label{sec:method:ordinalencoding}

\vspace{-0.2cm} 

We introduce an \emph{ordinal-aware} mixed-type encoding that maps each respondent's survey record to an initial embedding $\mathbf z_i^{(0)}\in\mathbb R^p$, where $p$ is the embedding dimension, used by the diffusion-based imputation model in Section~\ref{sec:method:diffusion_em}. Continuous variables indexed by \(j\in\mathcal J_c\) 
are standardized as $\tilde{x}_{ij}=(x_{ij}-\mu_j)/\sigma_j$,  where $\mu_j$ is the sample mean and $\sigma_j$ the sample standard deviation. Nominal variables indexed by \(j\in\mathcal J_n\) with \(K_j\) categories use analog-bit codewords~\cite{chen2023analogbits}: category \(k\in\{1,\ldots,K_j\}\) is mapped to the \(B_j\)-bit representation of \(k-1\),
$q_j(k)\in\{0,1\}^{B_j}$, where $B_j=\max\{1,\lceil\log_2K_j\rceil\}$. We write the encoded value of \(x_{ij}\) as \(q_j(x_{ij})\).

Ordinal variables  use a cumulative-probit scalar encoding
when their empirical category distribution is sufficiently stable. For  an ordinal variable indexed by \(j\in\mathcal J_o\), let $-\infty=b_{j,0}<b_{j,1}<\cdots<b_{j,k-1}<b_{j,k}<\cdots<b_{j,K_j}=+\infty$ define ordered cutpoints, so category $k\in\{1,\ldots,K_j\}$ corresponds to the interval $(b_{j,k-1},b_{j,k}]$. An observed $x_{ij}=k$ is represented by a scalar latent value $u_{ij}\in(b_{j,k-1},b_{j,k}]$, written as $u_{ij}=c_j(x_{ij})$. Given an imputed latent value $\hat u_{ij}$, the category probabilities are
\begin{equation}
\label{eq:probit_decoder}
P(x_{ij}=k\mid \hat u_{ij})
=
\Phi(b_{j,k}-\hat u_{ij})
-
\Phi(b_{j,k-1}-\hat u_{ij}),
\end{equation}
where $\Phi$ is the standard Gaussian cumulative distribution function (CDF). However, cumulative probit cutpoints can be unstable when observed responses are highly concentrated in a single category. To detect this case, we compute the empirical non-skip category proportions $\widehat p_{jk}$ for each ordinal variable indexed by $j$ and define $d_j=\max_k\widehat p_{jk}$ as the dominant-category proportion. We then route ordinal variables according to
\begin{equation}
\label{eq:routing}
\operatorname{enc}_j(x_{ij})=
\begin{cases}
q_j(x_{ij}), & d_j\geq \tau_{\mathrm{route}},\\
c_j(x_{ij}), & d_j< \tau_{\mathrm{route}},
\end{cases}
\qquad \text{with}~~
\tau_{\mathrm{route}}=0.70.
\end{equation}
Thus, ordinal variables with at least $70\%$ of observed non-skip entries in one category are encoded by analog bits; the rest use the cumulative probit encoder.
Let
\[
\mathcal J_o^{\mathrm{fb}}
=
\{j\in\mathcal J_o:d_j\geq\tau_{\mathrm{route}}\},
\qquad
\mathcal J_o^{\mathrm{sc}}
=
\{j\in\mathcal J_o:d_j<\tau_{\mathrm{route}}\},
\]
where $\mathcal J_o^{\mathrm{fb}}$ denotes the index set of fallback analog-bit ordinal
variables, and $\mathcal J_o^{\mathrm{sc}}$ denotes the index set  of scalar cumulative-probit
ordinal variables. The initial embedding for respondent $i$ concatenates the transformed blocks across the continuous, nominal, and ordinal variables:
\begin{equation}
\label{eq:encoding}
\mathbf z_i^{(0)}
=
\bigl(
\{\tilde{x}_{ij}:j\in\mathcal J_c\},
\{q_j(x_{ij}):j\in\mathcal J_n\cup\mathcal J_o^{\mathrm{fb}}\},
\{u_{ij}:j\in\mathcal J_o^{\mathrm{sc}}\}
\bigr).
\end{equation}
The corresponding embedding dimension is
\begin{equation}
\label{eq:total-encoding-dim}
p
=
|\mathcal J_c|
+
|\mathcal J_o^{\mathrm{sc}}| +
\sum_{j\in\mathcal J_n\cup\mathcal J_o^{\mathrm{fb}}} B_j
.
\end{equation}
Stacking the initial respondent embeddings gives $\mathbf Z^{(0)}=(\mathbf z_1^{(0)},\ldots,\mathbf z_N^{(0)})^\top\in\mathbb R^{N\times p}$. Appendix~\ref{app:ord_loss}
gives the remaining encoding details, including cutpoint definitions.

\subsection{Skip- and Ordinal-Informed Diffusion Imputation via Expectation-Maximization}
\label{sec:method:diffusion_em}

\vspace{-0.1cm} 

TabSODA inherits the EM-based diffusion strategy of DiffPuter~\cite{zhang2025diffputer} but applies it to the initial encoded survey matrix $\mathbf Z^{(0)}\in\mathbb R^{N\times p}$ and modifies training and sampling to retain structural skips and ordinal structure: while DiffPuter merges $S$ and $T$ into a single missing state, TabSODA keeps $S_{ij}=1$ as a fixed inapplicability state throughout preprocessing, training, sampling, and evaluation (see the full procedure in Algorithm~\ref{alg:tabsoda}). At EM iteration $r$, $\mathbf Z^{(r)}=(\mathbf z_1^{(r)},\ldots,\mathbf z_N^{(r)})^\top\in\mathbb R^{N\times p}$ denotes the current completed encoded table, initialized from $\mathbf Z^{(0)}$ in Section~\ref{sec:method:ordinalencoding}. We use $O$ for the training-visible mask and $M$ for the applicable missing mask, with $O=R$ and $M=T$ at deployment, and with $O_{ij}=R_{ij}(1-A_{ij})$ and  $M_{ij}=T_{ij}+R_{ij}A_{ij}$ under artificial hold-out mask $A$ for evaluation; structural-skip cells are excluded from both. Because diffusion operates in the encoded space, survey-level masks are broadcast from $d$ to $p$ dimensions, giving per-respondent encoded masks $\mathbf o_i^\star,\mathbf m_i^\star,\mathbf s_i^\star\in\{0,1\}^p$ for training-visible observed entries, applicable missing entries, and structural skips, with $\mathbf w_i^\star=\mathbf 1-\mathbf s_i^\star$ marking non-skip entries.

Let $\mathbf z_{i,\mathrm{obs}}$ and $\mathbf z_{i,\mathrm{mis}}$ denote entries indexed by $\mathbf o_i^\star=1$ and $\mathbf m_i^\star=1$; entries with $\mathbf s_i^\star=1$ are excluded from the complete-data likelihood $p_\theta(\mathbf z_i\mid\mathbf s_i^\star)$. The ideal E-step updates missing entries by the conditional expectation $\mathbb E[\mathbf z_{i,\mathrm{mis}}\mid\mathbf z_{i,\mathrm{obs}};\theta^{(r)}]$, and the M-step updates $\theta^{(r+1)}$ by maximizing the resulting complete-data log-likelihood (see equations in Appendix~\ref{app:heun})~\cite{zhang2025diffputer}. TabSODA approximates the E-step with conditional reverse diffusion and the M-step with EDM denoising regression. In the encoded space, TabSODA is a variance-exploding score-based diffusion model~\cite{song2021scoresde} with forward kernel $\mathbf z_t\mid\mathbf z_0\sim\mathcal N(\mathbf z_0,\sigma(t)^2 I)$, trained by denoising score matching~\cite{vincent2011dsm}. Under the EDM parameterization~\cite{karras2022edm}, $D_\theta(\mathbf z_t,\sigma)$ estimates the clean encoded vector $\mathbf z_0$ from its noisy version, with score estimate $s_\theta(\mathbf z_t,\sigma)=\{D_\theta(\mathbf z_t,\sigma)-\mathbf z_t\}/\sigma^2$.

\paragraph{E-step.}
Given \(\theta^{(r)}\), applicable missing entries are updated conditional on observed entries. At reverse step \(\tau_m\to\tau_{m+1}\), observed entries are refreshed by the forward Gaussian kernel, applicable missing entries are updated by one reverse-sampler step, and structural skips are clamped:
\begin{align}
\mathbf z_{i,\mathrm{obs}}^{(\tau_{m+1})}
&=
\mathbf z_i^{(r)}+\tau_{m+1}\boldsymbol\xi_i,
\qquad
\boldsymbol\xi_i\sim\mathcal N(0,I_p),
\label{eq:obs_refresh}
\\
\mathbf z_{i,\mathrm{mis}}^{(\tau_{m+1})}
&=
\mathcal S_{\theta^{(r)}}\!\left(
\mathbf z_i^{(\tau_m)},\tau_m,\tau_{m+1},\mathbf o_i^\star,\mathbf s_i^\star
\right),
\label{eq:miss_update}
\\
\mathbf z_i^{(\tau_{m+1})}
&=
\mathbf w_i^\star\odot
\left[
(\mathbf 1-\mathbf m_i^\star)\odot
\mathbf z_{i,\mathrm{obs}}^{(\tau_{m+1})}
+
\mathbf m_i^\star\odot
\mathbf z_{i,\mathrm{mis}}^{(\tau_{m+1})}
\right].
\label{eq:mix}
\end{align}
Here \(\mathcal S_\theta\) denotes one Heun-discretized Langevin reverse-sampler step~\cite{song2021scoresde,karras2022edm}; Algorithm~\ref{alg:tabsoda} and Appendix~\ref{app:heun} give the details on the full algorithm, Heun-step and noise grid. DiffPuter~\cite{zhang2025diffputer} shows that the mixed forward-reverse construction samples from $p_\theta(\mathbf x_{\mathrm{mis}}\mid\mathbf x_{\mathrm{obs}})$ in the continuous-time limit. TabSODA applies the same construction in encoded space, conditions on \(\mathbf o_i^\star\) and \(\mathbf s_i^\star\), and clamps structural-skip entries at every reverse step.

\paragraph{M-step.}
Given a completed encoded table $\mathbf Z^{(r)}$, the M-step fits the EDM denoiser. For noise level $\sigma_i$ and $\boldsymbol\epsilon_i\sim\mathcal N(0,I_p)$, the skip-aware noisy input is
$\mathbf z_i^{(\sigma,r)}
=
\mathbf w_i^\star\odot
(\mathbf z_i^{(r)}+\sigma_i\boldsymbol\epsilon_i),
$
so Gaussian noise is applied only to non-skip entries. The skip-masked EDM loss is
\begin{equation}
  \label{eq:diffusion_loss}
  \mathcal{L}_{\mathrm{diff}}^{(r)}
  =
  \mathbb E_{i,\sigma,\boldsymbol\epsilon}
  \!\left[
  \frac{1}{\sum_m w_{im}^\star}
  \sum_m
  w_{im}^\star
  \omega(\sigma_i)
  \left\{
  D_{\theta,m}
  \left(
  \mathbf z_i^{(\sigma,r)},\sigma_i,\mathbf o_i^\star,\mathbf s_i^\star
  \right)
  -
  z_{im}^{(r)}
  \right\}^2
  \right].
\end{equation}
where \(m\) indexes encoded entries and
\(\omega(\sigma_i)=(\sigma_i^2+\sigma_{\mathrm{data}}^2)/(\sigma_i\sigma_{\mathrm{data}})^2\) is the EDM weighting function~\cite{karras2022edm}. Relative to DiffPuter~\cite{zhang2025diffputer}, Eq.~\eqref{eq:diffusion_loss} assigns zero reconstruction weight to structural skips and conditions the denoiser on \(\mathbf o_i^\star\) and \(\mathbf s_i^\star\).

TabSODA adds two low-noise decoding terms. For an ordinal variable indexed by  $j\in\mathcal J_o^{\mathrm{sc}}$, the denoiser output is decoded through Eq.~\eqref{eq:probit_decoder}. Let $\Omega_j=\{i:O_{ij}=1\}$ denote training-visible observed cells for the $j$-th variable. We minimize a cumulative-link negative log likelihood~\cite{mccullagh1980,christensen2025clm} plus a Ranked Probability Score (RPS) term~\cite{epstein1969rps}:
\begin{equation}
\label{eq:scalar_loss}
\mathcal L_{\mathrm{ord},j}
=
\frac{1}{|\Omega_j|}
\sum_{i\in\Omega_j}
\left[
\ell_{\mathrm{NLL},ij}
+
\lambda_{\mathrm{RPS}}\,\ell_{\mathrm{RPS},ij}
\right].
\end{equation}
The NLL term rewards the observed ordinal category, and the RPS term penalizes cumulative distribution error. Because scalar ordinal variables use one encoded entry while analog-bit variables use \(B_j\), scalar ordinal reconstruction uses the width-parity multiplier
\(\lambda_j^{\mathrm{width}}=\max\{1,\lceil\log_2K_j\rceil\}\). Appendix~\ref{app:ord_loss} gives the NLL, RPS, cutpoint initialization, and ordered cutpoint parameterization.

The full M-step objective is
\begin{equation}
\label{eq:tabsoda_loss}
\mathcal L_{\mathrm{TabSODA}}^{(r)}
=
\mathcal L_{\mathrm{diff}}^{(r)}
+
\lambda_{\mathrm{cont}}\mathcal L_{\mathrm{cont}}^{(r)}
+
\lambda_{\mathrm{nom}}\mathcal L_{\mathrm{nom}}^{(r)}
+
\lambda_{\mathrm{ord}}\sum_{j\in\mathcal J_o^{\mathrm{sc}}}
\mathcal L_{\mathrm{ord},j}^{(r)}.
\end{equation}
Here $\mathcal L_{\mathrm{cont}}^{(r)}$ supervises standardized continuous entries by the squared error, $\mathcal L_{\mathrm{nom}}^{(r)}$ is the prototype cross-entropy over analog-bit blocks (Appendix~\ref{app:ord_loss}), and $\mathcal L_{\mathrm{ord},j}^{(r)}$ is the cumulative-probit ordinal loss in  Eq.~\eqref{eq:scalar_loss}. After the final E-step, encoded rows are decoded to the survey scale by inverse standardization, nearest analog-bit codeword decoding, and cumulative-probit MAP decoding. Appendix~\ref{app:hyperparams} reports the auxiliary-weight grid, blockwise schedule, and decoding details.

\section{Experiments}
\label{sec:experiments}
\vspace{-0.1cm} 

\subsection{Datasets}
\label{sec:experiments:data}
\vspace{-0.2cm} 
We evaluate TabSODA on two mixed-type survey benchmarks. The Population Assessment of Tobacco and Health Study (PATH) benchmark~\cite{icpsrpathpuf} contains $N=32{,}320$ respondents and $d=57$ variables, consisting of 1 continuous, 32 nominal, and 24 ordinal variables. The National Survey on Drug Use and Health (NSDUH) benchmark~\cite{samhsansduhdata} contains $N=58{,}034$ respondents and $d=29$ variables, consisting of 6 continuous, 10 nominal, and 13 ordinal variables. These two benchmarks are used for end-to-end imputation evaluation under artificial missingness.

We also construct a separate skip-pattern evaluation set to assess the CART-based skip miner independently of imputation accuracy. This set consists of six PATH and NSDUH variable subsets comprising 50, 100, and 200 variables. These subsets are used only to compare the estimated skip mask $\widehat S$ against the codebook-provided skip mask $S$; they are not used in the end-to-end imputation experiments. Additional dataset construction and cleaning details are provided in Appendix~\ref{app:datasets}.

\subsection{Implementation Details}
\label{sec:experiments:implementation}

\vspace{-0.1cm} 

We compare TabSODA with its direct baseline, DiffPuter~\cite{zhang2025diffputer}, and five external mixed-type imputation or tabular synthesis methods: MICE~\cite{vanbuuren2011mice}, MissForest~\cite{stekhoven2012missforest}, TabCSDI~\cite{zheng2022tabcsdi}, TabDiff~\cite{shi2024tabdiff}, and TabSyn~\cite{zhang2023tabsyn}. We also evaluate three internal variants: \textbf{TabSODA}, which uses the codebook skip mask \(S\); \textbf{TabSODA\,+\,SKIP}, which uses the CART-mined skip mask \(\widehat S\); and \textbf{TabSODA-N}, which routes all categorical variables through analog-bit encoding.

All end-to-end imputation experiments use \(R=5\) outer EM iterations, \(N_{\mathrm{imp}}=20\) sampled imputations per iteration, and \(L=50\) paired replications. We set \(R=5\) and \(N_{\mathrm{imp}}=20\) following DiffPuter ablations, which report stable performance with 4--5 EM iterations and at least 10 Monte Carlo samples per iteration~\cite{zhang2025diffputer}. For each benchmark, we apply \(30\%\) artificial masking under MCAR, MAR, and MNAR, using shared seeds and masks across methods for paired comparisons. The MAR and MNAR generators use rank-based scores so that masking probabilities do not depend on arbitrary categorical codes; details are given in Appendix~\ref{app:impl:masking}. Unless stated otherwise, TabSODA uses the same  hyperparameter configuration across datasets and missingness mechanisms; Appendix~\ref{app:hyperparams} reports the denoiser, EDM schedule, M-step learning rates, freeze schedule, sampler grid, auxiliary losses, and temperature grid.

\subsection{Evaluation Metrics}
\label{sec:experiments:metrics}

\vspace{-0.2cm} 

We compute imputation accuracy only on artificially hidden observed cells, \(\{(i,j): A_{ij}^{(\ell)}=1,\ R_{ij}=1\}\); structural skips are excluded because they are not valid imputation targets. For continuous variables, we report Root Mean Square Error (RMSE) on standardized values. For categorical variables, we report overall categorical accuracy (Cat. Acc.), nominal accuracy (Nom. Acc.), and ordinal accuracy (Ord. Acc.). For ordinal variables, we also report Mean Absolute Categorical Error (MACE)~\cite{baccianella2009ordinal}, defined as the mean absolute difference between predicted and true ordinal ranks. All metrics are averaged over \(L=50\) paired replications. Appendix~\ref{app:metrics} gives the metric formulas, standardization details, evaluation sets \(\mathcal{E}_c^{(\ell)},\mathcal{E}_n^{(\ell)},\mathcal{E}_o^{(\ell)}\), and aggregation procedure.

For the skip miner, we evaluate the estimated mask $\widehat S$ against the codebook-provided mask $S$ using the cell-level precision, recall, F1, false positive rate, and accuracy. The estimated mask is fitted using raw missingness and questionnaire order only. The cell-level confusion counts and the closed-form expressions for these metrics are given in Appendix~\ref{app:metrics}.

\subsection{Results}
\label{sec:experiments:results}

\vspace{-0.1cm} 

We compare TabSODA against six baseline methods, including MICE~\cite{vanbuuren2011mice}, MissForest~\cite{stekhoven2012missforest}, TabCSDI~\cite{zheng2022tabcsdi}, DiffPuter~\cite{zhang2025diffputer}, TabSyn~\cite{zhang2023tabsyn}, and TabDiff~\cite{shi2024tabdiff}, on the PATH and NSDUH datasets under
the three artificial-masking mechanisms, MAR, MCAR, and MNAR, each at a $30\%$ target-column hide rate. Table~\ref{tab:benchmarks_mar_combined} reports the MAR results; the MCAR and MNAR results follow the same performance pattern and are deferred to Appendix~\ref{app:additional_results}. TabSODA achieves the strongest categorical and ordinal performance on both datasets, with the largest gains on the ordinal metrics: on PATH it reduces ordinal MACE by $23.7\%$ relative to the strongest baseline, and on NSDUH it likewise improves both ordinal accuracy and ordinal MACE over the strongest baseline. TabDiff is competitive with TabSODA on PATH's numeric RMSE but trails on all other categorical and ordinal metrics. The MCAR and MNAR tables in Appendix~\ref{app:additional_results} confirm the same pattern across mechanisms.

\begin{table}[t]
\centering
\caption{Imputation performance on PATH and NSDUH under $30\%$ MAR (mean $\pm$ SD).}
\label{tab:benchmarks_mar_combined}
\small
\resizebox{\textwidth}{!}{%
\begin{tabular}{llccccc}
\toprule
Dataset & Method
& Ord.\ MACE \(\downarrow\)
& Ord.\ Acc.\ \(\uparrow\)
& Cat.\ Acc.\ \(\uparrow\)
& Nom.\ Acc.\ \(\uparrow\)
& Num.\ RMSE \(\downarrow\) \\
\midrule
\multirow{7}{*}{PATH}
& MICE        & \(0.757 \pm 0.099\) & \(0.564 \pm 0.037\) & \(0.744 \pm 0.020\) & \(0.820 \pm 0.014\) & \(15.102 \pm 3.778\) \\
& MissForest  & \(0.596 \pm 0.083\) & \(0.658 \pm 0.031\) & \(0.811 \pm 0.017\) & \(0.876 \pm 0.012\) & \(15.623 \pm 3.687\) \\
& TabCSDI  & \(1.154 \pm 0.169\) & \(0.443 \pm 0.059\) & \(0.653 \pm 0.025\) & \(0.743 \pm 0.026\) & \(42.579 \pm 23.546\) \\
& DiffPuter   & \(0.683 \pm 0.121\) & \(0.622 \pm 0.044\) & \(0.794 \pm 0.026\) & \(0.869 \pm 0.019\) & \(15.808 \pm 3.370\) \\
& TabSyn      & \(0.677 \pm 0.134\) & \(0.578 \pm 0.039\) & \(0.786 \pm 0.028\) & \(0.873 \pm 0.028\) & \(17.430 \pm 5.460\) \\
& TabDiff     & \(0.720 \pm 0.102\) & \(0.579 \pm 0.007\) & \(0.775 \pm 0.014\) & \(0.845 \pm 0.012\) & \(\mathbf{13.016 \pm 4.654}\) \\
& TabSODA     & \cellcolor{lightgreen}\(\mathbf{0.455 \pm 0.087}\) & \cellcolor{lightgreen}\(\mathbf{0.703 \pm 0.044}\) & \cellcolor{lightgreen}\(\mathbf{0.900 \pm 0.018}\) & \cellcolor{lightgreen}\(\mathbf{0.983 \pm 0.006}\) & \(13.211 \pm 4.410\) \\
\addlinespace
\multirow{7}{*}{NSDUH}
& MICE        & \(0.824 \pm 0.165\) & \(0.578 \pm 0.059\) & \(0.503 \pm 0.048\) & \(0.397 \pm 0.067\) & \(58.073 \pm 15.594\) \\
& MissForest  & \(0.824 \pm 0.165\) & \(0.578 \pm 0.059\) & \(0.503 \pm 0.048\) & \(0.397 \pm 0.067\) & \(56.313 \pm 15.378\) \\
& TabCSDI  & \(0.436 \pm 0.067\) & \(0.688 \pm 0.039\) & \(0.690 \pm 0.027\) & \(0.692 \pm 0.047\) & \(71.867 \pm 16.756\) \\
& DiffPuter   & \(0.424 \pm 0.035\) & \(0.676 \pm 0.015\) & \(0.709 \pm 0.036\) & \(0.753 \pm 0.068\) & \(66.812 \pm 24.300\) \\
& TabSyn      & \(0.399 \pm 0.054\) & \(0.692 \pm 0.038\) & \(0.711 \pm 0.022\) & \(0.748 \pm 0.043\) & \(56.709 \pm 16.646\) \\
& TabDiff     & \(0.413 \pm 0.052\) & \(0.687 \pm 0.037\) & \(0.701 \pm 0.024\) & \(0.729 \pm 0.041\) & \(57.312 \pm 16.847\) \\
& TabSODA     & \cellcolor{lightgreen}\(\mathbf{0.362 \pm 0.041}\) & \cellcolor{lightgreen}\(\mathbf{0.715 \pm 0.034}\) & \cellcolor{lightgreen}\(\mathbf{0.739 \pm 0.020}\) & \cellcolor{lightgreen}\(\mathbf{0.779 \pm 0.036}\) & \cellcolor{lightgreen}\(\mathbf{53.114 \pm 16.101}\) \\
\bottomrule
\end{tabular}%
}
\end{table}

Table~\ref{tab:mar_ablation} reports the MAR ablation comparing TabSODA with the codebook-provided skip mask, TabSODA\,+\,SKIP using the estimated mask $\widehat{S}$, and TabSODA-N using analog-bit encoding for all categoricals. On NSDUH, TabSODA with the codebook-provided mask achieves the best ordinal MACE, ordinal accuracy, categorical accuracy, and numeric RMSE, while TabSODA\,+\,SKIP attains the highest nominal accuracy. On PATH, TabSODA-N wins on most metrics, indicating that the analog-bit fallback remains a strong choice when ordinal columns have lower cardinality. Importantly, the gap between TabSODA\,+\,SKIP and the codebook-supplied variant is small on every metric, showing that the estimated mask recovers most of the codebook-supplied benefit. The corresponding MCAR and MNAR ablations in Appendix~\ref{app:additional_results} display the same pattern.

\begin{table}[t]\medskip
\centering
\caption{Ablation under $30\%$ MAR (mean $\pm$ SD): TabSODA (with codebook-provided skip mask), TabSODA+SKIP (with estimated mask $\widehat{S}$), and TabSODA-N (analog-bit encoding for all categoricals).}
\label{tab:mar_ablation}
\small
\resizebox{\textwidth}{!}{%
\begin{tabular}{llccccc}
\toprule
Dataset & Method
& Ord.\ MACE \(\downarrow\)
& Ord.\ Acc.\ \(\uparrow\)
& Cat.\ Acc.\ \(\uparrow\)
& Nom.\ Acc.\ \(\uparrow\)
& Num.\ RMSE \(\downarrow\) \\
\midrule
\multirow{3}{*}{PATH}
& TabSODA-N        & \(0.472 \pm 0.090\) & \(\mathbf{0.717 \pm 0.040}\) & \(\mathbf{0.904 \pm 0.016}\) & \(0.982 \pm 0.006\) & \(\mathbf{13.054 \pm 4.261}\) \\
& TabSODA\,+\,SKIP & \(0.465 \pm 0.090\) & \(0.691 \pm 0.047\) & \(0.895 \pm 0.019\) & \(0.982 \pm 0.006\) & \(13.279 \pm 4.353\) \\
& TabSODA          & \cellcolor{lightgreen}\(\mathbf{0.455 \pm 0.087}\) & \(0.703 \pm 0.044\) & \(0.900 \pm 0.018\) & \cellcolor{lightgreen}\(\mathbf{0.983 \pm 0.006}\) & \(13.211 \pm 4.410\) \\
\addlinespace
\multirow{3}{*}{NSDUH}
& TabSODA-N        & \(0.379 \pm 0.045\) & \(0.714 \pm 0.033\) & \(0.730 \pm 0.019\) & \(0.761 \pm 0.041\) & \(53.828 \pm 16.069\) \\
& TabSODA\,+\,SKIP & \(0.499 \pm 0.093\) & \(0.697 \pm 0.023\) & \(0.737 \pm 0.019\) & \(\mathbf{0.848 \pm 0.027}\) & \(53.594 \pm 16.407\) \\
& TabSODA          & \cellcolor{lightgreen}\(\mathbf{0.362 \pm 0.041}\) & \cellcolor{lightgreen}\(\mathbf{0.715 \pm 0.034}\) & \cellcolor{lightgreen}\(\mathbf{0.739 \pm 0.020}\) & \(0.779 \pm 0.036\) & \cellcolor{lightgreen}\(\mathbf{53.114 \pm 16.101}\) \\
\bottomrule
\end{tabular}%
}
\end{table}

Table~\ref{tab:miner_evaluation} evaluates the skip-pattern miner independently from the imputation model. The mined mask attains near-perfect precision on PATH and high precision on NSDUH at every variable scale, with the false positive rate staying under $3\%$ across NSDUH and well below $1\%$ across PATH. These results show that the decision tree miner recovers a high-precision approximation to the codebook-provided skip mask from raw missingness and questionnaire order alone, which explains why TabSODA\,+\,SKIP recovers most of the codebook-supplied benefit in Table~\ref{tab:mar_ablation}.

\begin{table}[h]\medskip
\centering
\footnotesize
\setlength{\tabcolsep}{3pt}
\renewcommand{\arraystretch}{0.92}
\caption{Skip pattern miner evaluation on PATH and NSDUH variable subsets at the 50, 100, and 200 variable scales. PATH: \(N=32{,}320\); NSDUH: \(N=58{,}034\).}
\label{tab:miner_evaluation}
\begin{tabular}{@{}lrrrrrrrr@{}}
\toprule
Dataset & \#Vars & \#Mined & \#Rules & Prec. & Rec. & F1 & FPR & Acc. \\
\midrule
\multirow{3}{*}{PATH}
& 50  & 48  & 35  & 0.999 & 0.991 & 0.995 & 0.001 & 0.996 \\
& 100 & 87  & 61  & 0.999 & 0.895 & 0.944 & 0.001 & 0.941 \\
& 200 & 153 & 101 & 0.998 & 0.942 & 0.969 & 0.003 & 0.961 \\
\midrule
\multirow{3}{*}{NSDUH}
& 50  & 50  & 9   & 0.957 & 0.852 & 0.901 & 0.022 & 0.932 \\
& 100 & 99  & 29  & 0.943 & 0.789 & 0.859 & 0.028 & 0.905 \\
& 200 & 198 & 70  & 0.948 & 0.843 & 0.892 & 0.023 & 0.932 \\
\bottomrule
\end{tabular}
\end{table}

\section{Conclusion}
\label{sec:conclusion}

\vspace{-0.2cm} 

We presented \textbf{TabSODA}, a diffusion-based imputer for tabular survey data that combines skip-aware diffusion with ordinal-aware encoding. Structural skips are propagated through the denoising loss, noisy EDM input, and reverse-time sampler, while ordinal variables are represented by cumulative-probit scalar latents with an analog-bit fallback. The structural-skip mask can be supplied by codebook, or estimated by \textbf{TabSODA\,+\,SKIP}, a CART-based miner that uses only raw missingness and questionnaire order. In practice, we recommend using the codebook mask when available and the mined mask when codebook rules are unavailable or costly to encode. Future work can extend TabSODA to longitudinal surveys by modeling wave-specific skip logic, time-varying ordinal responses, and respondent-level temporal dependence within the diffusion sampler.

\bibliographystyle{plainnat}

\bibliography{references}
\newpage
\appendix
\setcounter{page}{1}
\section{Skip-Mining Rule Screens and Implementation Details}
\label{app:skip_mining}

This appendix gives the rule-acceptance details for the skip-pattern miner in
Section~\ref{sec:method:skip_learning}. For target item \(q_t\), CART proposes
candidate rules~\cite{breiman1984cart}. A rule \(r\) induces a mask
\(r_{it}\in\{0,1\}\), where \(r_{it}=1\) means respondent \(i\) satisfies the
rule path. Rules are scored against the raw-missing label
\(y_{it}=\mathbb I\{x_{it}\text{ is raw-missing}\}\), not against the
codebook-provided skip mask \(S_{it}\), which is unavailable during mining. Figure~\ref{fig:skip_detection_flow} summarizes the one-target learning flow.

\begin{figure}[h]
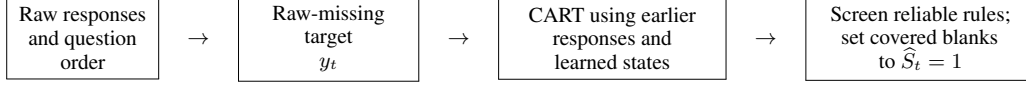

\centering
\small
\setlength{\fboxsep}{4pt}
\resizebox{\columnwidth}{!}{%
\begin{tabular}{c c c c c c c}
\fbox{\parbox{0.14\textwidth}{\centering Raw responses\\and question order}}
& $\rightarrow$ &
\fbox{\parbox{0.17\textwidth}{\centering Raw-missing target\\$y_t$}}
& $\rightarrow$ &
\fbox{\parbox{0.22\textwidth}{\centering CART using earlier\\responses and\\learned states}}
& $\rightarrow$ &
\fbox{\parbox{0.22\textwidth}{\centering Screen reliable rules;\\set covered blanks\\to $\widehat S_t=1$}}
\\
\end{tabular}%
}
\caption{Skip learning for one target item \(q_t\). The miner predicts the
raw-missing label \(y_t\) using only earlier questionnaire information. Accepted
rules promote covered raw blanks to structural skips, producing \(\widehat S_t\).}
\label{fig:skip_detection_flow}
\end{figure}

\begin{table}[h]
\centering
\small
\caption{Structural skip pattern induced by a gating question (first column). Participants who responded `No' are asked to skip the follow-up questions, and move to the next valid question.}
\label{tab:structural_skip_example}
\renewcommand{\arraystretch}{1.1}
\resizebox{\columnwidth}{!}{%
\begin{tabular}{c|cccc|c}
\toprule
Lifetime cigarette use & Current smoking & Last smoked & Recency detail & Daily quantity & Lifetime e-cigarette use \\
\midrule
\cellcolor{lightgreen}asked (Y/N) & \cellcolor{skipgray}skip & \cellcolor{skipgray}skip & \cellcolor{skipgray}skip & \cellcolor{skipgray}skip & \cellcolor{lightgreen}asked \\
\bottomrule
\end{tabular}%
}
\end{table}

\paragraph{Rule counts and screening rates.}
Let \(\mathcal I_t\) be the rows used to score a candidate rule for \(q_t\). The
rule table is
\[
\begin{array}{c|cc}
 & y_{it}=1 & y_{it}=0 \\
\hline
r_{it}=1 & \mathrm{TP}_t & \mathrm{FP}_t \\
r_{it}=0 & \mathrm{FN}_t & \mathrm{TN}_t
\end{array},
\qquad
\begin{aligned}
\mathrm{TP}_t &= \sum_{i\in\mathcal I_t}\mathbb I\{r_{it}=1,y_{it}=1\},&
\mathrm{FP}_t &= \sum_{i\in\mathcal I_t}\mathbb I\{r_{it}=1,y_{it}=0\},\\
\mathrm{FN}_t &= \sum_{i\in\mathcal I_t}\mathbb I\{r_{it}=0,y_{it}=1\},&
\mathrm{TN}_t &= \sum_{i\in\mathcal I_t}\mathbb I\{r_{it}=0,y_{it}=0\}.
\end{aligned}
\]
Thus, \(\mathrm{TP}_t\) counts covered raw blanks, \(\mathrm{FP}_t\) covered
nonblank cells, \(\mathrm{FN}_t\) missed raw blanks, and \(\mathrm{TN}_t\)
nonblank cells not covered by the rule. Define
\[
n_+(t)=\mathrm{TP}_t+\mathrm{FN}_t,\qquad
n_-(t)=\mathrm{FP}_t+\mathrm{TN}_t,\qquad
\widehat{\mathrm{FPR}}_t=\frac{\mathrm{FP}_t}{n_-(t)},\qquad
\widehat{\mathrm{FNR}}_t=\frac{\mathrm{FN}_t}{n_+(t)} .
\]
Here \(n_+(t)\) is the number of raw-missing target rows and \(n_-(t)\) is the
number of nonblank target rows. The false-positive rate measures rule firing on
nonblank targets, while the false-negative rate measures missed raw blanks.
Targets with \(n_+(t)=0\) or \(n_-(t)=0\) are excluded from rule acceptance.

\paragraph{Position-adjusted caps and support.}
Let \(d\) be the number of questionnaire columns and let \(\operatorname{pos}(t)\)
be the zero-based file-order position of \(q_t\). Define
\[
\rho(t)=\frac{\operatorname{pos}(t)}{d}.
\]
The value \(\rho(t)\) is near \(0\) for early items and near \(1\) for late
items. Since late follow-up items often apply to fewer respondents, the miner
uses
\begin{equation}
  \label{eq:miner_caps}
  f_{\mathrm{FPR}}^{(t)}
  =
  \eta_{\mathrm{FPR}}^{0}\{1+\alpha_{\mathrm{FPR}}\rho(t)\},
  \qquad
  f_{\mathrm{FNR}}^{(t)}
  =
  \eta_{\mathrm{FNR}}^{0}\{1+\alpha_{\mathrm{FNR}}\rho(t)\},
\end{equation}
where \(\eta_{\mathrm{FPR}}^{0}\) and \(\eta_{\mathrm{FNR}}^{0}\) are base error
caps, and \(\alpha_{\mathrm{FPR}}\) and \(\alpha_{\mathrm{FNR}}\) determine the
late-item relaxation. The false-negative cap also includes a small positive-class
correction,
\begin{equation}
  \label{eq:miner_fnr_cap}
  f_{\mathrm{FNR,cap}}^{(t)}
  =
  \min\left\{
  f_{\mathrm{FNR}}^{(t)} + n_+(t)^{-1/2},
  \eta_{\mathrm{FNR}}^{\mathrm{ceil}}
  \right\}.
\end{equation}
The \(n_+(t)^{-1/2}\) term reduces rejection of clean rules for sparse targets,
and \(\eta_{\mathrm{FNR}}^{\mathrm{ceil}}\) limits the total relaxation.

The target-specific support threshold is
\begin{equation}
  \label{eq:miner_support}
  n_{\min}^{(t)}
  =
  \max\left\{
  \left\lfloor \beta_n n_{\min}^{\mathrm{abs}}\right\rfloor,\,
  \left\lfloor n_{\min}\{1-\alpha_n\rho(t)\}\right\rfloor
  \right\},
  \qquad
  n_{\min}
  =
  \max\left\{n_{\min}^{\mathrm{abs}},\,\lceil \phi_{\min}N\rceil\right\}.
\end{equation}
Here \(N\) is the number of respondents, \(n_{\min}^{\mathrm{abs}}\) is an
absolute minimum, \(\phi_{\min}\) is a minimum support fraction, \(\alpha_n\)
relaxes support for later items, and \(\beta_n\) sets the lower floor relative
to \(n_{\min}^{\mathrm{abs}}\).

\paragraph{Wilson-bound screen.}
For an error count \(e\) among \(n\) relevant rows, let \(\widehat p=e/n\). With
\(z=1.96\), the Wilson upper confidence bound is
\begin{equation}
  \label{eq:wilson_ucb}
  \mathrm{UCB}(e,n)
  =
  \frac{\widehat p+z^2/(2n)}{1+z^2/n}
  +
  \frac{z}{1+z^2/n}
  \sqrt{\frac{\widehat p(1-\widehat p)+z^2/(4n)}{n}}.
\end{equation}
A candidate rule is eligible for the multiple-testing screen only if
\begin{equation}
  \label{eq:skip_rule_screen}
  \mathrm{UCB}\{\mathrm{FP}_t,n_-(t)\}
  \le
  f_{\mathrm{FPR}}^{(t)}
  \quad\text{and}\quad
  \mathrm{UCB}\{\mathrm{FN}_t,n_+(t)\}
  \le
  f_{\mathrm{FNR,cap}}^{(t)} .
\end{equation}
The first inequality controls firing on nonblank target cells; the second
controls missed raw blank target cells while counting finite-sample uncertainty.

\paragraph{Fisher test, FDR screen, and rule application.}
Rules passing \eqref{eq:skip_rule_screen} are tested for association with raw
target missingness using Fisher's exact test~\cite{fisher1934exact} on the
\(2\times2\) table above. Benjamini--Hochberg FDR control~\cite{benjamini1995fdr}
is then applied to the resulting \(p\)-values with threshold \(q_{\mathrm{BH}}\).
Let \(\mathcal R_t\) be the accepted rule set for \(q_t\). The mined skip mask is
\[
\widehat S_{it}
=
\mathbb I\left\{
x_{it}\text{ is raw-missing}
\;\text{and}\;
\exists r\in\mathcal R_t\text{ such that }r_{it}=1
\right\}.
\]
Thus, accepted rules can promote only raw blank cells to structural skips;
observed cells are never overwritten.

\paragraph{Neighboring-target extension and duplicate removal.}
A high-confidence rule for \(q_t\) may extend to nearby targets within eight
questionnaire columns in either direction when the same prediction mask explains
a block of follow-up items. For an extension target, the same mask must pass
relaxed checks with FPR and FNR caps \(1.5\,\eta_{\mathrm{FPR}}^{0}\) and
\(2.0\,\eta_{\mathrm{FNR}}^{0}\), respectively, and must have support on both
sides of the prediction mask. Duplicate rules with the same target set,
orientation, logic type, and prediction mask are collapsed, keeping the rule
with the smaller false-positive Wilson upper bound; ties are broken by larger
coverage.

\paragraph{Hyperparameters.}
The reported implementation uses maximum CART depth \(D_{\max}=4\), class
weights \(W=\{0\!:\!2,1\!:\!1\}\), minimum absolute leaf size
\(n_{\min}^{\mathrm{abs}}=30\), and BH threshold \(q_{\mathrm{BH}}=0.01\). The
shared position-adjustment constants and false-negative ceiling are
\[
(\alpha_{\mathrm{FPR}},\alpha_{\mathrm{FNR}},\alpha_n,\beta_n)
=
(0.2,0.5,0.25,0.75),
\qquad
\eta_{\mathrm{FNR}}^{\mathrm{ceil}}=0.05.
\]
The survey-level settings
\((\phi_{\min},\eta_{\mathrm{FPR}}^{0},\eta_{\mathrm{FNR}}^{0})\) are
\[
(0.02,0.007,0.02)
\quad\text{for PATH, NSDUH-50, and NSDUH-100,}
\qquad
(0.005,0.007,0.05)
\quad\text{for NSDUH-200.}
\]
These operating points are fixed before downstream imputation evaluation and are
not selected using imputation accuracy.

\section{Augmented Ordinal Loss and Cutpoint Initialization}
\label{app:ord_loss}

The cumulative-link NLL evaluated at latent $a$ is
\begin{equation}
  \label{eq:nll}
  \ell_{\mathrm{NLL},ij}(a,b_j)=-\log\max\{\Phi(b_{j,x_{ij}}-a)-\Phi(b_{j,x_{ij}-1}-a),\epsilon_\pi\},
\end{equation}
with $\epsilon_\pi=10^{-8}$. The RPS term \cite{epstein1969rps} sums squared cumulative residuals
\begin{equation}
  \label{eq:rps}
  \ell_{\mathrm{RPS},ij}(a,b_j)=\frac{1}{K_j-1}\sum_{k=1}^{K_j-1}(\widehat{F}_{ijk}(a,b_j)-H_{ijk})^2,
\end{equation}
where $\widehat{F}_{ijk}=\sum_{c\leq k}P(x_{ij}=c\mid a,b_j)$ and $H_{ijk}=\ind\{x_{ij}\leq k\}$. The promoted configuration uses uniform per-cell weighting (class balancing is disabled).

\paragraph{Softplus cutpoint reparameterization.} $b_{j,1}=\alpha_j$, and for $k\geq 2$, $b_{j,k}=b_{j,k-1}+\mathrm{softplus}(\delta_{j,k})+\epsilon_b$, with $\epsilon_b=10^{-4}$. This guarantees $b_{j,1}<\cdots<b_{j,K_j-1}$ for any $\alpha_j,\delta_{j,k}\in\R$.

\paragraph{Smoothed empirical-CDF init.} Cutpoints $b_j^{(0)}$ are initialized at $r=0$ from
\begin{equation}
  \label{eq:cutpoint_init}
  \widehat{F}_{j,\mathrm{smooth}}(k)=\frac{n_j\widehat{F}_{j,\mathrm{emp}}(k)+\alpha_{\mathrm{prior}}\,k/K_j}{n_j+\alpha_{\mathrm{prior}}},
  \quad
  b_{j,k}^{(0)}=\Phi^{-1}(\mathrm{clip}\{\widehat{F}_{j,\mathrm{smooth}}(k),10^{-4},1-10^{-4}\}),
\end{equation}
with $\alpha_{\mathrm{prior}}=5$. Eq.~\eqref{eq:cutpoint_init} is invoked only at $r=0$; subsequent iterations warm-start from the previous post-calibrated values.


The encoded state $\mathbf{z}_i^{(0)}\in\R^p$ defined in Eq.~\eqref{eq:encoding} of the main text concatenates standardized continuous entries $\tilde{x}_{ij}=(x_{ij}-\mu_j)/\sigma_j$ (with training-visible moments), analog-bit codewords $q_j(k)_b=\mathrm{bit}_b(k-1)$ for nominal and fallback-routed ordinal columns ($K_j$-category variables map to $\{0,1\}^{B_j}$ with $B_j=\max\{1,\lceil\log_2 K_j\rceil\}$ \cite{chen2023analogbits}), and scalar cumulative-probit latent entries $u_{ij}$ for scalar-routed ordinal columns. After concatenation, all non-skip entries pass through a coordinate-wise affine standardization $T_\Sigma$ that is refreshed across EM iterations under one of two policies (Appendix~\ref{app:hyperparams}): \emph{once} (fit on the initial completed table and held fixed thereafter, stabilizing the EDM denoising target across iterations when continuous columns dominate) or \emph{each\_iter} (recomputed at the start of every outer iteration, tracking the evolving ordinal-latent distribution when categorical columns dominate).

The ordinal-route partition $\mathcal{J}_o=\mathcal{J}_o^{\mathrm{sc}}\cup\mathcal{J}_o^{\mathrm{fb}}$ in Eq.~\eqref{eq:routing} is determined by the dominant category share $d_j=\max_k\widehat{p}_{jk}$, with $\tau_{\mathrm{route}}=0.70$: columns with $d_j\geq\tau_{\mathrm{route}}$ are near-degenerate and routed to the analog-bit fallback because their cutpoint estimates are unstable, while columns with $d_j<\tau_{\mathrm{route}}$ are routed to the cumulative-probit scalar latent.

\paragraph{Cumulative-probit observation model.}
For $j\in\mathcal{J}_o^{\mathrm{sc}}$ with ordered cutpoints $-\infty=b_{j,0}<b_{j,1}<\cdots<b_{j,K_j-1}<b_{j,K_j}=+\infty$, the latent variable $u_{ij}$ encodes the observed category through $x_{ij}=k\iff b_{j,k-1}<u_{ij}\leq b_{j,k}$, and the decoder of Eq.~\eqref{eq:probit_decoder} produces $P(x_{ij}=k\mid \hat{u}_{ij}) = \Phi(b_{j,k}-\hat{u}_{ij})-\Phi(b_{j,k-1}-\hat{u}_{ij})$. Cutpoints are guaranteed ordered through a softplus reparameterization \cite{wiemann2024softplus,christensen2025clm}, initialized at $r=0$ from a smoothed empirical CDF (see Eq.~\eqref{eq:cutpoint_init}), and warm-started across EM iterations.

\paragraph{Observed-cell encoder: warm-start schedule.}
The training-side encoder $g_j(k;b_j)$ uses a deterministic interval mean for the initial completion ($r=0$) and a stochastic truncated-normal draw thereafter:
\begin{equation}
  \label{eq:warmstart_schedule}
  u_{ij}^{\mathrm{train},(r)} =
  \begin{cases}
    \bar{g}_j(x_{ij};b_j^{(0)})=\E[U\mid b_{j,k-1}<U\leq b_{j,k},U\sim\mathcal{N}(0,1)], & r=0,\\
    \mathcal{N}(0,1)\text{ truncated to }(b_{j,x_{ij}-1},b_{j,x_{ij}}], & r\geq 1.
  \end{cases}
\end{equation}
The $r\geq 1$ branch is the diffusion-time analog of Albert--Chib data augmentation \cite{albertchib1993}: stochastic latent draws propagate within-category uncertainty into subsequent cutpoint and denoiser updates, while the $r=0$ branch supplies a stable, rank-preserving target before the denoiser has signal. Missing cells use a column-median initialization at $r=0$ \cite[\S 4.3]{littlerubin2002}.

\paragraph{Augmented ordinal loss and width parity.}
The per-column ordinal loss $\mathcal{L}_{\mathrm{ord},j}$ defined in Eq.~\eqref{eq:scalar_loss} of the main text is a uniform-weighted cumulative-link NLL plus a Ranked Probability Score \cite{epstein1969rps} term, averaged per column:
\[
  \mathcal{L}_{\mathrm{ord},j}
  = \frac{1}{|\Omega_j|}\sum_{i\in\Omega_j}\bigl[\ell_{\mathrm{NLL},ij}+\lambda_{\mathrm{RPS}}\,\ell_{\mathrm{RPS},ij}\bigr].
\]
The class-balanced variant of the loss (with inverse-frequency weights $\omega_{j,x_{ij}}$) is implemented but is disabled in our promoted configuration. A single scalar-routed ordinal coordinate would receive less gradient mass than a $B_j$-bit nominal block; we compensate with a width-parity multiplier $\lambda_j^{\mathrm{width}}=\max\{1,\lceil\log_2 K_j\rceil\}$ on the EDM reconstruction weight.

\paragraph{Validated per-column temperature.}
At decode time, we rescale the cumulative-probit decoder by a per-column temperature $\tau_j^\star>0$:
\begin{equation}
  \label{eq:temperature_decoder}
  P(x_{ij}=k\mid \hat{u}_{ij},\tau_j) = \Phi\!\left(\tfrac{b_{j,k}-\hat{u}_{ij}}{\tau_j}\right)-\Phi\!\left(\tfrac{b_{j,k-1}-\hat{u}_{ij}}{\tau_j}\right).
\end{equation}
Unlike softmax temperature scaling \cite{guo2017calibration}, $\tau_j$ \emph{can} change MAP predictions because cumulative-probit class probabilities are differences of CDFs rather than monotone transforms of a logit vector. We fit $\tau_j^\star$ on a held-out cell-level validation pool ($\rho_{\mathrm{val}}=0.20$ per column, removed from M-step training signal) by lexicographic minimization over a fixed temperature grid; the grid, the lexicographic selection key, the validation construction, and the limit expansions are in Appendix~\ref{app:temperature}.

\paragraph{Nominal prototype cross-entropy.}
Let $\hat{\mathbf z}_{ij}\in\R^{B_j}$ denote the low-noise denoiser output restricted to the analog-bit block of nominal column $j\in\mathcal{J}_n$, and let $\{q_j(k)\}_{k=1}^{K_j}\!\subset\!\{0,1\}^{B_j}$ be the valid codewords from Section~\ref{sec:method:ordinalencoding}. The prototype-softmax probability assigned to category $k$ is
\begin{equation}
\label{eq:nom_softmax}
p_\theta(x_{ij}\!=\!k\mid\hat{\mathbf z}_{ij})
=
\frac{\exp\!\bigl(-\|\hat{\mathbf z}_{ij}-q_j(k)\|_2^{2}\bigr)}{\sum_{k'=1}^{K_j}\exp\!\bigl(-\|\hat{\mathbf z}_{ij}-q_j(k')\|_2^{2}\bigr)},
\end{equation}
and $\mathcal L_{\mathrm{nom}}^{(r)}$ is the negative log-likelihood of the observed category averaged over training-visible cells,
\begin{equation}
\label{eq:nom_loss}
\mathcal L_{\mathrm{nom}}^{(r)}
=
-\sum_{j\in\mathcal J_n}
\frac{1}{|\Omega_j|}
\sum_{i\in\Omega_j}
\log p_\theta\!\bigl(x_{ij}\,\big|\,\hat{\mathbf z}_{ij}\bigr).
\end{equation}

\section{Validated Per-Column Temperature Search}
\label{app:temperature}

\paragraph{Validation construction.} For each scalar-routed ordinal column $j$, sample a cell-level validation pool $V_j\subseteq\Omega_j$ of size $\lceil\rho_{\mathrm{val}}|\Omega_j|\rceil$, where $\Omega_j=\{i:O_{ij}=1\}$ are the training-visible cells (Section~\ref{sec:method:diffusion_em}). Holdout cells are removed from all M-step training signals for column $j$ (encoded entries, observed-mask channel, ordinal-target lookup overwritten as missing). After the final outer iteration, the denoiser at noise $\sigma_{\mathrm{cal}}$ produces $\hat{u}_{ij}^{\mathrm{cal}}$; the pairs $(\hat{u}_{ij}^{\mathrm{cal}},x_{ij})_{i\in V_j}$ form the held-out validation set.

\paragraph{Selection key.} For each $\tau\in\mathcal{T}=\{0.5,0.7,1.0,1.4,2.0\}$ and each column $j$, compute on $V_j$ the accuracy $\mathrm{ACC}_j(\tau)$, MACE $\mathrm{MACE}_j(\tau)=|V_j|^{-1}\sum_i|\hat{x}_{ij}(\tau)-x_{ij}|$, and clipped NLL $\mathrm{NLL}_j(\tau)$. Select $\tau_j^\star$ by minimization of $K_j(\tau)=(-\mathrm{ACC}_j,\mathrm{MACE}_j,\mathrm{NLL}_j,|\tau-1|)$. The grid contains $\tau=1$, so the unscaled decoder is always available as fallback.

\paragraph{Why temperature can change the argmax.} Softmax temperature scaling preserves argmax \cite{guo2017calibration}; the cumulative-probit decoder does not, because per-class probabilities are CDF differences rather than monotone transforms of a logit vector. As $\tau\to 0^+$, all mass concentrates on the interval containing $\hat{u}_{ij}$; as $\tau\to\infty$, $P(x_{ij}=k\mid\hat{u}_{ij},\tau)=(b_{j,k}-b_{j,k-1})\phi(0)\tau^{-1}+O(\tau^{-3})$ for interior categories, so the limiting argmax is the widest-interval category, independent of $\hat{u}_{ij}$. Per-column temperature is therefore an accuracy lever in the cumulative-probit setting, not just a calibration knob.

\section{Denoiser Architecture}
\label{app:arch}

The TabSODA denoiser is a state-conditioned MLP wrapped in EDM preconditioning \cite{karras2022edm}, mirroring the released DiffPuter implementation \cite{zhang2025diffputer} and extending its conditioning channel from the observed mask alone to the joint observed and structural-skip masks $(\mathbf{o}_i^\star,\mathbf{s}_i^\star)$ of Eq.~\eqref{eq:tabsoda_loss}. Let $p$ be the encoded dimension (Section~\ref{sec:method:ordinalencoding}), $H$ the hidden width, and $\sigma$ the noise level.

\paragraph{EDM preconditioning.} The full denoiser is
\begin{equation}
  \label{eq:precond}
  D_\theta(\mathbf{z},\sigma,\mathbf{o}_i^\star,\mathbf{s}_i^\star) = \kappa_{\mathrm{skip}}(\sigma)\,\mathbf{z} + \kappa_{\mathrm{out}}(\sigma)\,F_\theta\!\bigl(\kappa_{\mathrm{in}}(\sigma)\,\mathbf{z},\,\kappa_{\mathrm{noise}}(\sigma),\,\mathbf{o}_i^\star,\,\mathbf{s}_i^\star\bigr),
\end{equation}
with the standard EDM scalings \cite{karras2022edm}
\[
\kappa_{\mathrm{skip}} = \frac{\sigma_{\mathrm{data}}^2}{\sigma^2+\sigma_{\mathrm{data}}^2},
\quad
\kappa_{\mathrm{out}} = \frac{\sigma\,\sigma_{\mathrm{data}}}{\sqrt{\sigma^2+\sigma_{\mathrm{data}}^2}},
\quad
\kappa_{\mathrm{in}} = \frac{1}{\sqrt{\sigma^2+\sigma_{\mathrm{data}}^2}},
\quad
\kappa_{\mathrm{noise}}(\sigma) = \tfrac{1}{4}\log\sigma,
\]
and $\sigma_{\mathrm{data}}=0.5$ (Appendix~\ref{app:hyperparams}).

\paragraph{Backbone $F_\theta$.} The backbone is a feedforward SiLU MLP that takes the preconditioned input $\kappa_{\mathrm{in}}\,\mathbf{z}\in\R^p$ concatenated with the conditioning masks $\mathbf{o}_i^\star,\mathbf{s}_i^\star\in\{0,1\}^p$ (jointly $2p$ extra channels), projected to width $H$ and combined additively with the time embedding before three hidden Linear$\to$SiLU pairs and a final linear readout:
\begin{align}
  h_0 &= W_{\mathrm{proj}}\,[\kappa_{\mathrm{in}}\,\mathbf{z};\, \mathbf{o}_i^\star;\, \mathbf{s}_i^\star] + \phi(\kappa_{\mathrm{noise}}(\sigma)) \in \R^{H}, \\
  h_1 &= \mathrm{SiLU}(W_1 h_0)\in\R^{2H}, \quad
  h_2 = \mathrm{SiLU}(W_2 h_1)\in\R^{2H}, \\
  h_3 &= \mathrm{SiLU}(W_3 h_2)\in\R^{H}, \quad
  F_\theta = W_4 h_3 \in\R^{p}.
\end{align}
The block widths $H{\to}2H{\to}2H{\to}H$ follow the released DiffPuter denoiser; only the input projection $W_{\mathrm{proj}}\in\R^{H\times 3p}$ differs from DiffPuter, by $p$ extra columns to absorb the structural-skip mask channel.

\paragraph{Time embedding $\phi$.} The scalar $\kappa_{\mathrm{noise}}(\sigma)$ is mapped to $\R^H$ by a sinusoidal positional embedding with $H/2$ logarithmically spaced frequencies $\omega_k=(1/10000)^{k/(H/2)}$ for $k=0,\ldots,H/2-1$, returning $[\cos(\omega_k\kappa_{\mathrm{noise}})\,\|\,\sin(\omega_k\kappa_{\mathrm{noise}})]$ with the cosine and sine halves swapped (matching the DiffPuter implementation). The result is passed through a two-layer SiLU MLP $\phi:\R^H\to\R^H$ before being added to $W_{\mathrm{proj}}\,[\kappa_{\mathrm{in}}\,\mathbf{z};\,\mathbf{o}_i^\star;\,\mathbf{s}_i^\star]$.

\paragraph{Width selection.} The promoted width $H=1024$ was selected by an internal sweep over $H\in\{256,512,768,1024\}$ on PATH: ordinal accuracy improved monotonically up through $1024$ with diminishing returns, and per-replication wall-clock grew sub-linearly because the dominant cost is E-step sampling rather than backbone forward passes. The sinusoidal-embedding base $10000$ and the EDM scalings are kept at the values used by the released DiffPuter code so that the only architectural differences between TabSODA and DiffPuter are (i) the $+p$-column conditioning expansion and (ii) the encoded-dimension change induced by routing scalar ordinal columns through a single latent variable instead of a $\lceil\log_2 K_j\rceil$-bit codeword.

\section{EDM Langevin SDE Sampler (E-Step)}
\label{app:heun}

\paragraph{Generic EM updates.}
With complete-data density $p_\theta(\mathbf z_i\mid\mathbf s_i^\star)$, the ideal E-step (Section~\ref{sec:method:diffusion_em})~\cite{zhang2025diffputer} updates missing entries by the conditional expectation
\begin{equation}\label{eq:em_estep}
\mathbf z_{i,\mathrm{mis}}^{(r+1)} \approx \mathbb E\bigl[\mathbf z_{i,\mathrm{mis}}\mid\mathbf z_{i,\mathrm{obs}};\theta^{(r)}\bigr],
\end{equation}
and the M-step updates the parameters by maximizing the resulting complete-data log-likelihood,
\begin{equation}\label{eq:em_mstep}
\theta^{(r+1)} = \argmax_\theta \sum_{i=1}^N \log p_\theta\bigl(\mathbf z_{i,\mathrm{obs}},\mathbf z_{i,\mathrm{mis}}^{(r+1)}\,\big|\,\mathbf s_i^\star\bigr).
\end{equation}
TabSODA approximates the E-step in Eq.~\eqref{eq:em_estep} by conditional reverse diffusion through Eqs.~\eqref{eq:obs_refresh}--\eqref{eq:mix} and the M-step in Eq.~\eqref{eq:em_mstep} by the EDM denoising regression of Eq.~\eqref{eq:diffusion_loss}.

Following Karras et al.\ \cite[\S 6]{karras2022edm}, the reverse-time SDE we sample from decomposes into a probability-flow ODE plus a Langevin diffusion term whose stationary marginal at fixed $t$ is the model's $p_t(\mathbf{z})$. Each reverse step from $\tau_m$ to $\tau_{m+1}$ first inflates the noise level by a churn step ($\hat{\tau}_m=(1+\gamma_m)\tau_m$ with $\gamma_m=\min(S_{\mathrm{churn}}/M,\sqrt{2}-1)$) and injects fresh Wiener noise $\hat{\mathbf{z}}_i=\mathbf{z}_i+\sqrt{\hat{\tau}_m^2-\tau_m^2}\,S_{\mathrm{noise}}\boldsymbol{\eta}_i$, then applies a second-order Heun integrator to transport the state from $\hat{\tau}_m$ down to $\tau_{m+1}$. The Heun corrector is $\mathbf{z}_{m+1}=\hat{\mathbf{z}}+\tfrac12(\tau_{m+1}-\hat{\tau}_m)(\mathbf{d}_m+\mathbf{d}'_m)$ with $\mathbf{d}_m=\bigl(\hat{\mathbf{z}}-D_\theta(\hat{\mathbf{z}},\hat{\tau}_m,\mathbf{o}_i^\star,\mathbf{s}_i^\star)\bigr)/\hat{\tau}_m$, conditioning the denoiser on the encoded observed and structural-skip masks. We use $S_{\mathrm{churn}}=S_{\mathrm{noise}}=1$ to recover the consistent (clean) Langevin form and a single mixed-update cycle ($N_{\mathrm{refresh}}=1$) per reverse step. After the final reverse step, missing non-skip ordinal entries are decoded to MAP categories through the temperature-scaled cumulative-probit decoder of Eq.~\eqref{eq:temperature_decoder}; nominal and fallback ordinal columns decode through the analog-bit codeword route.

\paragraph{E-step sampler details.}
The reverse-time noise grid follows the EDM polynomial schedule
\[
\tau_m=
\left(
\sigma_{\max}^{1/\rho}
+
\frac{m}{M-1}
\{\sigma_{\min}^{1/\rho}-\sigma_{\max}^{1/\rho}\}
\right)^\rho,
\qquad
m=0,\ldots,M-1,\quad \rho=7.
\]
The sampler step \(\mathcal S_\theta\) in Eq.~\eqref{eq:miss_update} combines a churn-step Euler--Maruyama update for the Langevin correction with a second-order Heun integrator for the probability-flow ODE~\cite{song2021scoresde,karras2022edm}. At each step, observed entries are refreshed by Eq.~\eqref{eq:obs_refresh}, applicable missing entries are updated by Eq.~\eqref{eq:miss_update}, and structural-skip entries are clamped through Eq.~\eqref{eq:mix}.
Algorithm~\ref{alg:tabsoda} is the full outer EM loop.

\begin{algorithm}[h]
\caption{TabSODA: outer EM training loop with skip-aware M-step and skip-clamped E-step.}
\label{alg:tabsoda}
\begin{algorithmic}[1]
\Require Raw table $\mathbf{X}\in(\mathbb{R}\cup\{\textsc{na}\})^{N\times d}$, codebook-provided or estimated skip mask $S\in\{0,1\}^{N\times d}$, type sets $\mathcal{J}_c,\mathcal{J}_n,\mathcal{J}_o$, outer iterations $R$, reverse steps $M$
\Statex \textbf{Notation:} $A,O,M\!\in\!\{0,1\}^{N\times d}$ are the (optional) hiding, training-visible, and applicable-missing survey masks (Section~\ref{sec:method:diffusion_em}); $\mathbf{o}_i^\star,\mathbf{s}_i^\star,\mathbf{m}_i^\star\!\in\!\{0,1\}^p$ are the encoded broadcasts of $O_i,S_i,M_i$ and $\mathbf{w}_i^\star=\mathbf{1}-\mathbf{s}_i^\star$; $\mathbf{Z}^{(r)}=(\mathbf{z}_1^{(r)},\ldots,\mathbf{z}_N^{(r)})^\top\!\in\!\mathbb{R}^{N\times p}$ is the current completed encoded table; $D_\theta$ is the EDM-preconditioned denoiser conditioned on $(\mathbf{o}^\star,\mathbf{s}^\star)$; $\mathcal{S}_\theta$ is one Heun-discretized Langevin reverse-sampler step (Eq.~\eqref{eq:miss_update}); $T_\Sigma$ is the coordinate-wise standardizer, fit per the refresh policy of Appendix~\ref{app:hyperparams}; $\tau_0\!=\!\sigma_{\max},\tau_M\!=\!\sigma_{\min}$.
\Statex
\State Optionally sample artificial hiding mask $A$ (Appendix~\ref{app:impl:masking}); set $O_{ij}\!=\!R_{ij}(1-A_{ij})$, $M_{ij}\!=\!T_{ij}+R_{ij}A_{ij}$ (default at deployment: $O\!=\!R$, $M\!=\!T$ when $A\!=\!0$)
\State Broadcast survey masks to encoded space to obtain $\mathbf{o}_i^\star,\mathbf{s}_i^\star,\mathbf{m}_i^\star$, and set $\mathbf{w}_i^\star\!=\!\mathbf{1}-\mathbf{s}_i^\star$
\State Route each $j\!\in\!\mathcal{J}_o$ into scalar route $\mathcal{J}_o^{\mathrm{sc}}$ or analog-bit fallback $\mathcal{J}_o^{\mathrm{fb}}$ via Eq.~\eqref{eq:routing}; initialize cutpoints $b^{(0)}$ via Eq.~\eqref{eq:cutpoint_init}
\State Build initial encoded table $\mathbf{Z}^{(0)}$: observed cells via Eq.~\eqref{eq:encoding}; applicable-missing cells filled by column means (median for $\mathcal{J}_o^{\mathrm{sc}}$); structural-skip cells set to $0$
\State Fit standardizer $T_\Sigma$ on $\{\mathbf{z}_i^{(0)}\}$ over entries with $\mathbf{w}_i^\star\!=\!1$; refresh policy per Appendix~\ref{app:hyperparams} (\emph{once}: freeze for all $r$; \emph{each iter}: recompute at the start of every $r$)
\For{$r=1,\ldots,R$}
  \Statex \quad\textbf{M-step (blockwise calibration; Section~\ref{sec:method:diffusion_em}):}
  \If{refresh policy is \emph{each\_iter}}\State Refit $T_\Sigma$ on $\{\mathbf{z}_i^{(r-1)}\}$ over $\mathbf{w}_i^\star\!=\!1$ entries\EndIf
  \State \textbf{Phase 1: pre-calibrate cutpoints.}\; $b^{(r,\mathrm{pre})}\!\leftarrow\!\argmin_b \sum_{j\in\mathcal{J}_o^{\mathrm{sc}}}\mathcal{L}_{\mathrm{ord},j}\!\bigl(b\!\mid\!\theta^{(r-1)}\bigr)$ \quad(Eq.~\eqref{eq:scalar_loss})
  \State \textbf{Phase 2: update denoiser with cutpoints frozen.}\; $\theta^{(r)}\!\leftarrow\!\argmin_\theta\mathcal{L}_{\mathrm{TabSODA}}^{(r)}\!\bigl(\theta\!\mid\!b^{(r,\mathrm{pre})}\bigr)$ \quad(Eq.~\eqref{eq:tabsoda_loss})
  \State \textbf{Phase 3: post-calibrate cutpoints.}\; $b^{(r)}\!\leftarrow\!\argmin_b \sum_{j\in\mathcal{J}_o^{\mathrm{sc}}}\mathcal{L}_{\mathrm{ord},j}\!\bigl(b\!\mid\!\theta^{(r)}\bigr)$
  \Statex \quad\textbf{E-step (skip-clamped reverse sampling):}
  \For{each respondent $i=1,\ldots,N$}
    \State Initialize $\mathbf{z}_i^{(\tau_0)}\!\sim\!\mathcal{N}(\mathbf{0},\tau_0^2 I_p)$; clamp skip entries: $\mathbf{z}_i^{(\tau_0)}\!\leftarrow\!\mathbf{w}_i^\star\!\odot\!\mathbf{z}_i^{(\tau_0)}$
    \For{$m=0,\ldots,M-1$}
      \State Refresh observed entries: $\mathbf{z}_{i,\mathrm{obs}}^{(\tau_{m+1})}\!=\!\mathbf{z}_i^{(r)}+\tau_{m+1}\boldsymbol{\xi}_i$, $\boldsymbol{\xi}_i\!\sim\!\mathcal{N}(\mathbf{0},I_p)$ \quad(Eq.~\eqref{eq:obs_refresh})
      \State Update missing entries by Heun-discretized Langevin step: $\mathbf{z}_{i,\mathrm{mis}}^{(\tau_{m+1})}\!=\!\mathcal{S}_{\theta^{(r)}}\!\bigl(\mathbf{z}_i^{(\tau_m)},\tau_m,\tau_{m+1},\mathbf{o}_i^\star,\mathbf{s}_i^\star\bigr)$ \quad(Eq.~\eqref{eq:miss_update}; Appendix~\ref{app:heun})
      \State Mix observed and missing channels with skip clamp: $\mathbf{z}_i^{(\tau_{m+1})}\!=\!\mathbf{w}_i^\star\!\odot\!\bigl[(\mathbf{1}-\mathbf{m}_i^\star)\!\odot\!\mathbf{z}_{i,\mathrm{obs}}^{(\tau_{m+1})}+\mathbf{m}_i^\star\!\odot\!\mathbf{z}_{i,\mathrm{mis}}^{(\tau_{m+1})}\bigr]$ \quad(Eq.~\eqref{eq:mix})
      \State Clamp skip entries: $\mathbf{z}_i^{(\tau_{m+1})}\!\leftarrow\!\mathbf{w}_i^\star\!\odot\!\mathbf{z}_i^{(\tau_{m+1})}$
    \EndFor
    \State Update completed table $\mathbf{z}_i^{(r+1)}$: continuous entries from $\mathbf{z}_i^{(\tau_M)}$; nominal and $\mathcal{J}_o^{\mathrm{fb}}$ via analog-bit decoding; $\mathcal{J}_o^{\mathrm{sc}}$ via cumulative-probit MAP at $\tau_j\!=\!1$ (Eq.~\eqref{eq:probit_decoder})
  \EndFor
\EndFor
\Statex \textbf{Decode-time temperature search (Section~\ref{sec:method:ordinalencoding}):}
\State Fit per-column temperatures $\tau_j^\star$ for $j\!\in\!\mathcal{J}_o^{\mathrm{sc}}$ by lexicographic minimization on the held-out validation pool $V_j$ (Appendix~\ref{app:temperature})
\State Decode missing $\mathcal{J}_o^{\mathrm{sc}}$ cells via temperature-scaled MAP (Eq.~\eqref{eq:temperature_decoder})
\State \Return Imputed matrix $\hat{\mathbf{X}}$
\end{algorithmic}
\end{algorithm}

\section{Default Hyperparameters}
\label{app:hyperparams}

Table~\ref{tab:hyperparams} lists the TabSODA configuration. Most values are shared across PATH and NSDUH; entries that differ between the two benchmarks are listed in separate columns. The promoted configuration uses uniform per-cell weighting in the ordinal auxiliary loss (i.e., class balancing is disabled).

\begin{table}[h]
\centering
\small
\caption{TabSODA hyperparameter configuration. Single-column rows apply to both PATH and NSDUH; two-column rows give the dataset-specific value.}
\label{tab:hyperparams}
\begin{tabular}{lccp{0.40\textwidth}}
\toprule
Parameter & PATH & NSDUH & Description \\
\midrule
\multicolumn{4}{l}{\textit{EM and replication}} \\
$R$ & \multicolumn{2}{c}{$5$} & Outer EM iterations \\
$R_{\min}$ & \multicolumn{2}{c}{$2$} & Minimum accepted outer iterations \\
$N_{\mathrm{imp}}$ & \multicolumn{2}{c}{$20$} & Sampled imputations per outer iteration (Monte Carlo E-step) \\
$L$ & \multicolumn{2}{c}{$50$} & Evaluation replications \\
Pseudo-missing ramp & \multicolumn{2}{c}{$(0,0,0.25,0.5,0.75)$} & Per-iteration weights for imputed non-skip cells \\
$(\lambda_{\mathrm{cont}},\lambda_{\mathrm{nom}},\lambda_{\mathrm{ord}})$ & \multicolumn{2}{c}{$(0.25,1.0,3.0)$} & Auxiliary loss weights \\
$\lambda_{\mathrm{RPS}}$ & \multicolumn{2}{c}{$0.1$} & RPS weight in scalar ordinal loss \\
$\tau_{\mathrm{route}}$ & \multicolumn{2}{c}{$0.70$} & Hybrid ordinal routing threshold (extreme-nominal fallback) \\
$\sigma_{\mathrm{cal}}$ & \multicolumn{2}{c}{$0.02$} & Calibration noise level \\
\midrule
\multicolumn{4}{l}{\textit{Encoding and warmup}} \\
Standardizer refresh & each iter & once & Affine standardizer recomputed each EM iter (PATH) or fit once (NSDUH) \\
Width-parity boost & \multicolumn{2}{c}{$\lceil\log_2 K_j\rceil$} & $\lambda_j^{\mathrm{width}}$ on scalar ordinal entries \\
\midrule
\multicolumn{4}{l}{\textit{Diffusion backbone (EDM)}} \\
$H$ & $1024$ & $512$ & Denoiser width \\
$(P_{\mathrm{mean}},P_{\mathrm{std}},\sigma_{\mathrm{data}})$ & \multicolumn{2}{c}{$(-1.2,1.2,0.5)$} & Log-normal noise schedule, EDM data scale \\
Batch size, epochs, patience & \multicolumn{2}{c}{$4096$, $1000$, $100$} & Per-iter Adam optimizer settings \\
\midrule
\multicolumn{4}{l}{\textit{Blockwise M-step}} \\
$\eta_{\mathrm{diff}}$ & $10^{-4}$ & $5\!\times\!10^{-5}$ & Adam learning rate (denoiser) \\
$\eta_{\mathrm{cal}}$ & $10^{-3}$ & $5\!\times\!10^{-4}$ & Adam learning rate (cutpoints) \\
$\alpha_{\mathrm{prior}}$ & \multicolumn{2}{c}{$5$} & Cutpoint prior strength (Eq.~\eqref{eq:cutpoint_init}) \\
\midrule
\multicolumn{4}{l}{\textit{E-step Heun-discretized Langevin sampler}} \\
$M$ & \multicolumn{2}{c}{$50$} & Reverse steps \\
$N_{\mathrm{refresh}}$ & $4$ & $1$ & Refresh cycles per reverse step \\
$\sigma_{\max}$ & $20$ & $5$ & Sampler maximum noise \\
$\sigma_{\min},\rho$ & \multicolumn{2}{c}{$0.002$, $7$} & Sampler grid \\
$S_{\mathrm{churn}},S_{\mathrm{noise}}$ & \multicolumn{2}{c}{$0$, $1$} & Langevin injection rate, Wiener-increment scale \\
\midrule
\multicolumn{4}{l}{\textit{Validated per-column ordinal temperature}} \\
$\rho_{\mathrm{val}}$, $\mathcal{T}$ & \multicolumn{2}{c}{$0.20$, $\{0.5,0.7,1.0,1.4,2.0\}$} & Validation-pool fraction, temperature grid \\
\bottomrule
\end{tabular}
\end{table}

\paragraph{Auxiliary loss weights.}
The triplet $(\lambda_{\mathrm{cont}},\lambda_{\mathrm{nom}},\lambda_{\mathrm{ord}})=(0.25,1.0,3.0)$ is selected by a one-pass hyperparameter grid search; an inner cross-validation would be a natural alternative. The smaller continuous weight reflects that continuous variables are already supervised directly by the EDM reconstruction objective in Eq.~\eqref{eq:diffusion_loss}, whereas nominal and ordinal variables require low-noise decoder supervision because good imputation depends not only on reconstructing an encoded coordinate but also on mapping that coordinate to a valid category or ordered interval. The width-parity multiplier $\lambda_j^{\mathrm{width}}=\max\{1,\lceil\log_2 K_j\rceil\}$ corrects the EDM-loss contribution of scalar-routed ordinal variables relative to fixed-width nominal encodings; the larger ordinal auxiliary weight $\lambda_{\mathrm{ord}}=3.0$ separately strengthens the cumulative-link decoder supervision and the cutpoint-sensitive ordinal likelihood during the M-step. Thus $\lambda_{\mathrm{ord}}$ should be viewed as a decoder-calibration stabilization weight, not as the mechanism correcting encoded-width imbalance.

\paragraph{Standardizer refresh policy.}
Diffusion-based tabular imputers typically standardize continuous features once at preprocessing and freeze the standardizer across iterations~\cite{zhang2023tabsyn}. TabSODA exposes an adaptive alternative, denoted \emph{each iter}, that recomputes the column-wise standardizer $T_\Sigma$ at the start of each EM outer iteration. The trade-off depends on the continuous-feature share of the dataset: when continuous columns dominate the standardizer's input space, adaptive refresh introduces harmful non-stationarity into the diffusion training distribution and hurts denoiser calibration; when categorical features dominate and continuous columns are sparse, adaptive refresh enables the model to track the evolving ordinal-latent distribution at minimal cost to continuous-channel calibration. We adopt a simple selection rule, computable from the dataset schema alone:
\[
  \text{standardizer refresh}
  =
  \begin{cases}
    \emph{once}, & \text{if continuous columns constitute }\geq 10\%\text{ of features,}\\
    \emph{each iter}, & \text{otherwise}.
  \end{cases}
\]
On our two benchmarks, PATH has $1/57 \approx 1.8\%$ continuous columns, well below the threshold, so we use \emph{each iter}; NSDUH has $6/29 \approx 20.7\%$ continuous columns, well above the threshold, so we use \emph{once}. Both datasets fall far from the $10\%$ boundary, and both empirically prefer the policy the rule selects.

\section{Datasets, Data Cleaning, and Artificial Masking}
\label{app:datasets}
\label{app:impl:masking}

\paragraph{PATH primary benchmark.}
The Population Assessment of Tobacco and Health (PATH) study \cite{icpsrpathpuf} is a nationally representative longitudinal cohort survey of tobacco-use behaviors, attitudes, and health outcomes among U.S.\ adults and youth, jointly funded by the National Institute on Drug Abuse (NIDA) at the National Institutes of Health (NIH) and the U.S.\ Food and Drug Administration's (FDA) Center for Tobacco Products, with public-use data distributed through the Inter-university Consortium for Political and Social Research (ICPSR). We use the Wave 1 Adult interview file and the data is publicly available at \href{https://www.icpsr.umich.edu/web/NAHDAP/studies/36498}{https://www.icpsr.umich.edu/web/NAHDAP/studies/36498}. The subset variables (1 continuous, 32 nominal, 24 ordinal), anchored on the cigarette, e-cigarette, cigar, pipe-tobacco, and smokeless-tobacco screener and follow-up sections (lifetime-use gates, current-frequency items, recency, daily quantity bins), augmented with nine same-domain ordinal smoking and age-band variables to densify the ordinal subset. Maximum number of categories per ordinal column is \(K_j\!\le\!9\); structural skips are concentrated in the cigarette/e-cigarette follow-up blocks gated by \texttt{AC1002}/\texttt{AE1001}.

\paragraph{NSDUH primary benchmark.}
The National Survey on Drug Use and Health (NSDUH) \cite{samhsansduhdata} is an annual nationally representative cross-sectional survey of U.S.\ residents aged 12 and older covering substance use, mental health, and related health behaviors, sponsored by the Substance Abuse and Mental Health Services Administration (SAMHSA) through its Center for Behavioral Health Statistics and Quality (CBHSQ). The data is publicly available at \href{https://www.samhsa.gov/data/data-we-collect/nsduh-national-survey-drug-use-and-health/datafiles}{https://www.samhsa.gov/data/data-we-collect/nsduh-national-survey-drug-use-and-health/datafiles} . The subset of  year 2021 data contains \(N=58{,}034\) respondents and \(d=29\) variables (6 continuous, 10 nominal, 13 ordinal) drawn from six domains for generalizability across health-policy applications: physical health, mental health (e.g., depression and anxiety screeners), substance-use frequency and recency, demographics, health-insurance coverage, and self-rated health and functional status. Ordinal columns are predominantly Likert-style frequency and severity scales with \(K_j\!\le\!7\).

\paragraph{Skip-rule miner evaluation set.}
Six smaller subsets drawn from the same two surveys are reserved exclusively for the codebook-reference evaluation of the forward-CART miner reported in Table~\ref{tab:miner_evaluation}: PATH-50, PATH-100, PATH-200 and NSDUH-50, NSDUH-100, NSDUH-200, where the suffix denotes the number of variables retained. These subsets share the same row dimensions as the corresponding primary benchmarks (\(N=32{,}320\) for PATH; \(N=58{,}034\) for NSDUH) but expose progressively wider variable cross-sections to stress-test the position-relaxed CART screen at different questionnaire breadths. They are not used in any end-to-end imputation experiment.

\paragraph{Data Cleaning.} Special codes (refused, don't know, missing-data, legitimate-skip) are recoded to \texttt{NA} in the imputation values file. The cleaning step does \emph{not} separate structural-skip codes from item-level nonresponse; the skip-pattern miner reconstructs that split. The codebook-provided skip indicator is extracted once and used either as the input mask $S$ in TabSODA or as the evaluation reference for Table~\ref{tab:miner_evaluation}.

\paragraph{Artificial masking generators.} For replication $\ell$, the artificial mask $A_{ij}^{(\ell)}$ is drawn from one of three mechanisms aligned with the DiffPuter family \cite{zhang2025diffputer} but operating on stable rank-based observed-value scores $s_{ij}\in[0,1]$ (numeric: empirical CDF rank; categorical: frequency rank divided by $K_j-1$). Variables are partitioned into a driver subset $\mathcal{D}$ ($\rho_{\mathrm{drv}}=0.30$) and target subset $\mathcal{V}$, fixed across MCAR/MAR/MNAR within a replication.
\begin{itemize}[leftmargin=2em,topsep=2pt,itemsep=2pt]
  \item \textbf{MCAR.} $A_{ij}^{(\ell)}\!\sim\!\mathrm{Bernoulli}(rR_{ij})$ on target columns, with $r=0.30$.
  \item \textbf{MAR.} $A_{ij}^{(\ell)}\!\sim\!\mathrm{Bernoulli}(R_{ij}\sigma(s_{i,\mathcal{D}}^\top\bm{\beta}_j+\alpha_j))$, with $\bm{\beta}_j\!\sim\!\mathcal{N}(0,I_{|\mathcal{D}|})$ rescaled to unit-variance logits and $\alpha_j$ fit by bisection so the marginal masking rate equals $r$. Logit depends only on never-masked driver scores, satisfying Rubin's MAR condition \cite{rubin1976,carpenter2021missing}.
  \item \textbf{MNAR.} Same logistic structure with (i) driver scores partially MCAR-zeroed at rate $r$ inside the mask generator and (ii) a target self-weight $w_j\!\sim\!\mathcal{N}(0,1)$ added so that $A_{ij}^{(\ell)}$ depends on the target value being hidden.
\end{itemize}
The same per-replication seed is shared across mechanisms, enabling paired comparisons.

\section{Evaluation Metric Formulas}
\label{app:metrics}

This appendix gives the formal definitions of the imputation and skip-miner metrics summarized in Section~\ref{sec:experiments:metrics}. All imputation metrics are computed only on artificially hidden, originally observed cells; structural-skip cells ($S_{ij}=1$) are excluded by construction.

\paragraph{Evaluation cells.}
For replication $\ell$, let
\begin{equation}
  \mathcal{E}^{(\ell)} = \{(i,j) : A_{ij}^{(\ell)}=1,\, R_{ij}=1\}
\end{equation}
denote the artificially hidden, originally observed cells, and let $\mathcal{E}_c^{(\ell)},\,\mathcal{E}_n^{(\ell)},\,\mathcal{E}_o^{(\ell)}$ be its restrictions to continuous, nominal, and ordinal columns ($j\in\mathcal{J}_c,\mathcal{J}_n,\mathcal{J}_o$), with $\mathcal{E}_{\mathrm{cat}}^{(\ell)}=\mathcal{E}_n^{(\ell)}\cup\mathcal{E}_o^{(\ell)}$.

\paragraph{Continuous variables (standardized MAE / RMSE).}
Continuous values are standardized with the per-column training mean $\mu_j$ and standard deviation $\sigma_j$, giving $\widetilde{x}_{ij}=(x_{ij}-\mu_j)/\sigma_j$ and the corresponding standardized prediction $\widetilde{\hat{x}}_{ij}=(\hat{x}_{ij}-\mu_j)/\sigma_j$. Per replication:
\begin{align}
  \mathrm{MAE}_{\mathrm{cont}}^{(\ell)} &= \frac{1}{|\mathcal{E}_c^{(\ell)}|}\!\!\sum_{(i,j)\in\mathcal{E}_c^{(\ell)}}\!\!|\widetilde{\hat{x}}_{ij}-\widetilde{x}_{ij}|, \\
  \mathrm{RMSE}_{\mathrm{cont}}^{(\ell)} &= \sqrt{\frac{1}{|\mathcal{E}_c^{(\ell)}|}\!\!\sum_{(i,j)\in\mathcal{E}_c^{(\ell)}}\!\!\bigl(\widetilde{\hat{x}}_{ij}-\widetilde{x}_{ij}\bigr)^{2}}.
\end{align}

\paragraph{Categorical accuracy (Cat./Nom./Ord. Acc.).}
With predicted categories $\hat{x}_{ij}\in\{1,\ldots,K_j\}$, the three reported accuracies differ only in the cell index set:
\begin{equation}
  \mathrm{ACC}_{\bullet}^{(\ell)} = \frac{1}{|\mathcal{E}_{\bullet}^{(\ell)}|}\sum_{(i,j)\in\mathcal{E}_{\bullet}^{(\ell)}}\mathbb{I}\{\hat{x}_{ij}=x_{ij}\},\qquad \bullet\in\{\mathrm{cat},\mathrm{nom},\mathrm{ord}\}.
\end{equation}

\paragraph{Ordinal MAE / MACE.}
Treating ordinal labels as integer ranks in $\{1,\ldots,K_j\}$,
\begin{equation}
  \mathrm{MACE}_{\mathrm{ord}}^{(\ell)} = \mathrm{MAE}_{\mathrm{ord}}^{(\ell)} = \frac{1}{|\mathcal{E}_o^{(\ell)}|}\sum_{(i,j)\in\mathcal{E}_o^{(\ell)}}|\hat{x}_{ij}-x_{ij}|;
\end{equation}
``MACE'' (mean absolute category error) and ``MAE'' are used interchangeably for ordinal variables across tables (the computation is identical). Intuitively, a MACE of $0$ corresponds to perfect ordinal prediction, while a MACE of $1$ means the prediction is on average one category off the true level. MACE is the standard ordinal-error measure preferred over plain accuracy because it credits near-miss predictions and penalizes large mis-rankings more heavily~\cite{baccianella2009ordinal}.

\paragraph{Aggregation across replications.}
Each per-replication metric $m^{(\ell)}$ is summarized as $\bar{m}\pm\mathrm{SE}(m)$ over the $L=50$ paired replications, where $\bar{m}=L^{-1}\sum_\ell m^{(\ell)}$ and $\mathrm{SE}(m)=\bigl(L(L-1)\bigr)^{-1/2}\bigl(\sum_\ell(m^{(\ell)}-\bar{m})^2\bigr)^{1/2}$.

\paragraph{Skip-rule miner metrics.}
The estimated mask $\widehat{S}\in\{0,1\}^{N\times d}$ is evaluated against the codebook-provided mask $S\in\{0,1\}^{N\times d}$ at the cell level. The cell-level confusion counts are
\begin{align}
  \mathrm{TP} &= \sum_{(i,j)}\mathbb{I}\{\widehat{S}_{ij}=1,\,S_{ij}=1\}, & \mathrm{FP} &= \sum_{(i,j)}\mathbb{I}\{\widehat{S}_{ij}=1,\,S_{ij}=0\},\\
  \mathrm{FN} &= \sum_{(i,j)}\mathbb{I}\{\widehat{S}_{ij}=0,\,S_{ij}=1\}, & \mathrm{TN} &= \sum_{(i,j)}\mathbb{I}\{\widehat{S}_{ij}=0,\,S_{ij}=0\}.
\end{align}
The reported audit metrics are
\begin{equation}
  \mathrm{Precision}=\frac{\mathrm{TP}}{\mathrm{TP}+\mathrm{FP}},\quad \mathrm{Recall}=\frac{\mathrm{TP}}{\mathrm{TP}+\mathrm{FN}},\quad F_{1}=\frac{2\,\mathrm{Precision}\cdot\mathrm{Recall}}{\mathrm{Precision}+\mathrm{Recall}},
\end{equation}
\begin{equation}
  \mathrm{FPR}=\frac{\mathrm{FP}}{\mathrm{FP}+\mathrm{TN}},\quad \mathrm{Accuracy}=\frac{\mathrm{TP}+\mathrm{TN}}{\mathrm{TP}+\mathrm{FP}+\mathrm{TN}+\mathrm{FN}}.
\end{equation}
$\widehat{S}$ is fitted from raw missingness and questionnaire order only; the codebook-provided mask $S$ is loaded after learning is complete and is used solely for these audit metrics.

\section{Additional Results: MCAR and MNAR Tables}
\label{app:additional_results}

This appendix collects the imputation benchmark (Tables~\ref{tab:benchmarks_mcar_combined} and~\ref{tab:benchmarks_mnar_combined}) and the skip-channel ablation (Tables~\ref{tab:mcar_ablation} and~\ref{tab:mnar_ablation}) under MCAR and MNAR masking. The qualitative ranking matches the MAR results in the main text: TabSODA attains the best Ord.\ MACE and Ord.\ Acc.\ on both datasets across both mechanisms, the best Cat.\ Acc.\ and Nom.\ Acc.\ in most settings, and the best Num.\ RMSE on NSDUH; PATH Num.\ RMSE is occasionally won by DiffPuter or TabDiff.

\begin{table}[h]
\centering
\caption{Imputation performance on PATH and NSDUH under $30\%$ MCAR (mean $\pm$ SD).}
\label{tab:benchmarks_mcar_combined}
\small
\resizebox{\textwidth}{!}{%
\begin{tabular}{llccccc}
\toprule
Dataset & Method
& Ord.\ MACE \(\downarrow\)
& Ord.\ Acc.\ \(\uparrow\)
& Cat.\ Acc.\ \(\uparrow\)
& Nom.\ Acc.\ \(\uparrow\)
& Num.\ RMSE \(\downarrow\) \\
\midrule
\multirow{7}{*}{PATH}
& MICE        & \(0.737 \pm 0.080\) & \(0.574 \pm 0.034\) & \(0.754 \pm 0.019\) & \(0.827 \pm 0.011\) & \(15.287 \pm 3.062\) \\
& MissForest  & \(0.584 \pm 0.072\) & \(0.664 \pm 0.030\) & \(0.816 \pm 0.016\) & \(0.878 \pm 0.009\) & \(15.562 \pm 3.177\) \\
& TabCSDI  & \(1.071 \pm 0.108\) & \(0.463 \pm 0.038\) & \(0.679 \pm 0.020\) & \(0.766 \pm 0.016\) & \(36.413 \pm 23.811\) \\
& DiffPuter   & \(0.656 \pm 0.107\) & \(0.634 \pm 0.042\) & \(0.801 \pm 0.037\) & \(0.869 \pm 0.036\) & \(\mathbf{15.010 \pm 2.602}\) \\
& TabSyn      & \(0.712 \pm 0.123\) & \(0.568 \pm 0.056\) & \(0.786 \pm 0.040\) & \(0.882 \pm 0.036\) & \(23.767 \pm 6.752\) \\
& TabDiff     & \(0.790 \pm 0.119\) & \(0.553 \pm 0.044\) & \(0.772 \pm 0.026\) & \(0.853 \pm 0.008\) & \(15.333 \pm 4.491\) \\
& TabSODA     & \cellcolor{lightgreen}\(\mathbf{0.452 \pm 0.063}\) & \cellcolor{lightgreen}\(\mathbf{0.703 \pm 0.045}\) & \cellcolor{lightgreen}\(\mathbf{0.900 \pm 0.022}\) & \cellcolor{lightgreen}\(\mathbf{0.983 \pm 0.007}\) & \(15.440 \pm 3.268\) \\
\addlinespace
\multirow{7}{*}{NSDUH}
& MICE        & \(0.797 \pm 0.149\) & \(0.586 \pm 0.054\) & \(0.499 \pm 0.037\) & \(0.396 \pm 0.061\) & \(55.200 \pm 15.684\) \\
& MissForest  & \(0.797 \pm 0.149\) & \(0.586 \pm 0.054\) & \(0.499 \pm 0.037\) & \(0.396 \pm 0.061\) & \(53.415 \pm 15.430\) \\
& TabCSDI  & \(0.437 \pm 0.060\) & \(0.685 \pm 0.040\) & \(0.689 \pm 0.033\) & \(0.695 \pm 0.051\) & \(61.857 \pm 16.716\) \\
& DiffPuter   & \(0.424 \pm 0.043\) & \(0.677 \pm 0.033\) & \(0.709 \pm 0.037\) & \(\mathbf{0.757 \pm 0.064}\) & \(55.301 \pm 21.712\) \\
& TabSyn      & \(0.413 \pm 0.012\) & \(0.684 \pm 0.011\) & \(0.693 \pm 0.025\) & \(0.704 \pm 0.046\) & \(\mathbf{42.380 \pm 34.830}\) \\
& TabDiff     & \(0.439 \pm 0.018\) & \(0.678 \pm 0.019\) & \(0.705 \pm 0.028\) & \(0.741 \pm 0.043\) & \(63.739 \pm 28.388\) \\
& TabSODA     & \cellcolor{lightgreen}\(\mathbf{0.390 \pm 0.040}\) & \cellcolor{lightgreen}\(\mathbf{0.693 \pm 0.035}\) & \cellcolor{lightgreen}\(\mathbf{0.716 \pm 0.034}\) & \(0.748 \pm 0.045\) & \(49.609 \pm 18.022\) \\
\bottomrule
\end{tabular}%
}
\end{table}

\begin{table}[h]
\centering
\caption{Imputation performance on PATH and NSDUH under $30\%$ MNAR (mean $\pm$ SD).}
\label{tab:benchmarks_mnar_combined}
\small
\resizebox{\textwidth}{!}{%
\begin{tabular}{llccccc}
\toprule
Dataset & Method
& Ord.\ MACE \(\downarrow\)
& Ord.\ Acc.\ \(\uparrow\)
& Cat.\ Acc.\ \(\uparrow\)
& Nom.\ Acc.\ \(\uparrow\)
& Num.\ RMSE \(\downarrow\) \\
\midrule
\multirow{7}{*}{PATH}
& MICE        & \(0.786 \pm 0.100\) & \(0.550 \pm 0.037\) & \(0.721 \pm 0.023\) & \(0.796 \pm 0.023\) & \(15.428 \pm 4.205\) \\
& MissForest  & \(0.611 \pm 0.088\) & \(0.648 \pm 0.035\) & \(0.789 \pm 0.021\) & \(0.850 \pm 0.021\) & \(16.241 \pm 3.951\) \\
& TabCSDI  & \(1.248 \pm 0.228\) & \(0.420 \pm 0.065\) & \(0.620 \pm 0.032\) & \(0.709 \pm 0.037\) & \(34.697 \pm 32.119\) \\
& DiffPuter   & \(0.724 \pm 0.161\) & \(0.609 \pm 0.059\) & \(0.764 \pm 0.051\) & \(0.834 \pm 0.050\) & \(\mathbf{14.895 \pm 4.806}\) \\
& TabSyn      & \(0.618 \pm 0.093\) & \(0.596 \pm 0.046\) & \(0.773 \pm 0.016\) & \(0.842 \pm 0.027\) & \(18.053 \pm 1.588\) \\
& TabDiff     & \(0.732 \pm 0.108\) & \(0.562 \pm 0.046\) & \(0.730 \pm 0.033\) & \(0.799 \pm 0.035\) & \(17.691 \pm 2.455\) \\
& TabSODA     & \cellcolor{lightgreen}\(\mathbf{0.433 \pm 0.064}\) & \cellcolor{lightgreen}\(\mathbf{0.706 \pm 0.036}\) & \cellcolor{lightgreen}\(\mathbf{0.901 \pm 0.017}\) & \cellcolor{lightgreen}\(\mathbf{0.981 \pm 0.009}\) & \(16.453 \pm 4.586\) \\
\addlinespace
\multirow{7}{*}{NSDUH}
& MICE        & \(0.824 \pm 0.171\) & \(0.578 \pm 0.072\) & \(0.505 \pm 0.053\) & \(0.399 \pm 0.081\) & \(62.851 \pm 19.202\) \\
& MissForest  & \(0.824 \pm 0.171\) & \(0.578 \pm 0.072\) & \(0.505 \pm 0.053\) & \(0.399 \pm 0.081\) & \(60.829 \pm 19.021\) \\
& TabCSDI  & \(0.430 \pm 0.096\) & \(0.686 \pm 0.064\) & \(0.680 \pm 0.057\) & \(0.675 \pm 0.071\) & \(61.723 \pm 23.300\) \\
& DiffPuter   & \(0.451 \pm 0.053\) & \(0.656 \pm 0.030\) & \(0.690 \pm 0.025\) & \(\mathbf{0.739 \pm 0.036}\) & \(58.711 \pm 20.712\) \\
& TabSyn      & \(0.417 \pm 0.081\) & \(0.679 \pm 0.054\) & \(0.691 \pm 0.045\) & \(0.716 \pm 0.057\) & \(57.712 \pm 20.013\) \\
& TabDiff     & \(0.433 \pm 0.085\) & \(0.673 \pm 0.058\) & \(0.682 \pm 0.047\) & \(0.701 \pm 0.052\) & \(61.736 \pm 18.075\) \\
& TabSODA     & \cellcolor{lightgreen}\(\mathbf{0.367 \pm 0.085}\) & \cellcolor{lightgreen}\(\mathbf{0.709 \pm 0.057}\) & \cellcolor{lightgreen}\(\mathbf{0.720 \pm 0.049}\) & \cellcolor{lightgreen}\(\mathbf{0.739 \pm 0.059}\) & \cellcolor{lightgreen}\(\mathbf{54.480 \pm 20.798}\) \\
\bottomrule
\end{tabular}%
}
\end{table}

\begin{table}[h]
\centering
\caption{Ablation under $30\%$ MCAR (mean $\pm$ SD): TabSODA with codebook-provided skip mask, TabSODA\,+\,SKIP with estimated mask $\widehat{S}$, and TabSODA-N (analog-bit encoding for all categoricals).}
\label{tab:mcar_ablation}
\small
\resizebox{\textwidth}{!}{%
\begin{tabular}{llccccc}
\toprule
Dataset & Method
& Ord.\ MACE \(\downarrow\)
& Ord.\ Acc.\ \(\uparrow\)
& Cat.\ Acc.\ \(\uparrow\)
& Nom.\ Acc.\ \(\uparrow\)
& Num.\ RMSE \(\downarrow\) \\
\midrule
\multirow{3}{*}{PATH}
& TabSODA-N        & \(0.479 \pm 0.087\) & \(\mathbf{0.710 \pm 0.045}\) & \(\mathbf{0.902 \pm 0.021}\) & \(0.983 \pm 0.007\) & \(15.507 \pm 3.242\) \\
& TabSODA\,+\,SKIP & \(0.471 \pm 0.064\) & \(0.691 \pm 0.042\) & \(0.896 \pm 0.021\) & \(0.982 \pm 0.006\) & \(\mathbf{15.423 \pm 3.236}\) \\
& TabSODA          & \cellcolor{lightgreen}\(\mathbf{0.452 \pm 0.063}\) & \(0.703 \pm 0.045\) & \(0.900 \pm 0.022\) & \cellcolor{lightgreen}\(\mathbf{0.983 \pm 0.007}\) & \(15.440 \pm 3.268\) \\
\addlinespace
\multirow{3}{*}{NSDUH}
& TabSODA-N        & \(0.418 \pm 0.046\) & \(0.691 \pm 0.034\) & \(0.708 \pm 0.034\) & \(0.732 \pm 0.046\) & \(50.543 \pm 18.020\) \\
& TabSODA\,+\,SKIP & \(0.548 \pm 0.089\) & \(0.671 \pm 0.033\) & \(0.716 \pm 0.031\) & \(\mathbf{0.847 \pm 0.036}\) & \(49.923 \pm 18.074\) \\
& TabSODA          & \cellcolor{lightgreen}\(\mathbf{0.390 \pm 0.040}\) & \cellcolor{lightgreen}\(\mathbf{0.693 \pm 0.035}\) & \cellcolor{lightgreen}\(\mathbf{0.716 \pm 0.034}\) & \(0.748 \pm 0.045\) & \cellcolor{lightgreen}\(\mathbf{49.609 \pm 18.022}\) \\
\bottomrule
\end{tabular}%
}
\end{table}

\begin{table}[h]
\centering
\caption{Ablation under $30\%$ MNAR (mean $\pm$ SD): TabSODA with codebook-provided skip mask, TabSODA\,+\,SKIP with estimated mask $\widehat{S}$, and TabSODA-N (analog-bit encoding for all categoricals).}
\label{tab:mnar_ablation}
\small
\resizebox{\textwidth}{!}{%
\begin{tabular}{llccccc}
\toprule
Dataset & Method
& Ord.\ MACE \(\downarrow\)
& Ord.\ Acc.\ \(\uparrow\)
& Cat.\ Acc.\ \(\uparrow\)
& Nom.\ Acc.\ \(\uparrow\)
& Num.\ RMSE \(\downarrow\) \\
\midrule
\multirow{3}{*}{PATH}
& TabSODA-N        & \(\mathbf{0.427 \pm 0.085}\) & \(\mathbf{0.734 \pm 0.035}\) & \(\mathbf{0.908 \pm 0.017}\) & \(0.981 \pm 0.009\) & \(16.621 \pm 4.579\) \\
& TabSODA\,+\,SKIP & \(0.438 \pm 0.058\) & \(0.704 \pm 0.036\) & \(0.898 \pm 0.018\) & \(0.979 \pm 0.010\) & \(16.501 \pm 4.554\) \\
& TabSODA          & \(0.433 \pm 0.064\) & \(0.706 \pm 0.036\) & \(0.901 \pm 0.017\) & \cellcolor{lightgreen}\(\mathbf{0.981 \pm 0.009}\) & \cellcolor{lightgreen}\(\mathbf{16.453 \pm 4.586}\) \\
\addlinespace
\multirow{3}{*}{NSDUH}
& TabSODA-N        & \(0.380 \pm 0.085\) & \(\mathbf{0.714 \pm 0.058}\) & \(0.719 \pm 0.047\) & \(0.732 \pm 0.054\) & \(54.648 \pm 20.426\) \\
& TabSODA\,+\,SKIP & \(0.509 \pm 0.110\) & \(0.684 \pm 0.053\) & \(\mathbf{0.720 \pm 0.049}\) & \(\mathbf{0.825 \pm 0.067}\) & \(55.117 \pm 21.918\) \\
& TabSODA          & \cellcolor{lightgreen}\(\mathbf{0.367 \pm 0.085}\) & \(0.709 \pm 0.057\) & \(0.720 \pm 0.049\) & \(0.739 \pm 0.059\) & \cellcolor{lightgreen}\(\mathbf{54.480 \pm 20.798}\) \\
\bottomrule
\end{tabular}%
}
\end{table}

\newpage

\section{Proof of Cutpoint Ordering Guarantee}
\label{app:cutpoint_proof}

\begin{proposition}
The softplus-reparameterized cutpoints satisfy $b_{j,1}<b_{j,2}<\cdots<b_{j,K_j-1}$ for any $\alpha_j\in\R$, $\delta_{j,k}\in\R$.
\end{proposition}
\begin{proof}
$b_{j,1}=\alpha_j$. For $k\geq 2$: $b_{j,k}=b_{j,k-1}+\mathrm{softplus}(\delta_{j,k})+\epsilon_b > b_{j,k-1}$ since both $\mathrm{softplus}(\cdot)>0$ and $\epsilon_b>0$. The strict inequality propagates by induction.
\end{proof}

\end{document}